\documentclass{article}

\usepackage[nonatbib,final]{neurips_2018}

\usepackage[utf8]{inputenc} 
\usepackage[T1]{fontenc}    
\usepackage{hyperref}       
\usepackage{url}            
\usepackage{booktabs}       
\usepackage{amsfonts}       
\usepackage{nicefrac}       
\usepackage{microtype}      
\usepackage{graphicx}
\usepackage{amsmath}
\usepackage{amssymb}
\usepackage{amsthm}
\usepackage{algorithm}
\usepackage{algorithmic}
\usepackage{color}
\usepackage{subcaption}
\usepackage{sgame, tikz}
\usepackage[algo2e, noend, noline, ruled, linesnumbered]{algorithm2e}

\providecommand{\DontPrintSemicolon}{\dontprintsemicolon}
\DontPrintSemicolon

\newcommand{\argmin}{\operatornamewithlimits{argmin}}
\newcommand{\argmax}{\operatornamewithlimits{argmax}}
\newcommand{\bE}{\mathbb{E}}
\newcommand{\bI}{\mathbb{I}}
\newcommand{\ba}{\mathbf{a}}
\newcommand{\bpi}{\mathbf{\pi}}

\newcommand{\bone}{\mathbf{1}}

\newcommand{\bx}{\mathbf{x}}
\newcommand{\by}{\mathbf{y}}
\newcommand{\bw}{\mathbf{w}}
\newcommand{\bv}{\mathbf{v}}
\newcommand{\btheta}{\boldsymbol\theta}
\newcommand{\cA}{\mathcal{A}}
\newcommand{\cB}{\mathcal{B}}
\newcommand{\cD}{\mathcal{D}}

\newcommand{\cH}{\mathcal{H}}

\newcommand{\cN}{\mathcal{N}}
\newcommand{\cO}{\mathcal{O}}

\newcommand{\cS}{\mathcal{S}}
\newcommand{\cT}{\mathcal{T}}
\newcommand{\cZ}{\mathcal{Z}}

\newcommand{\defword}[1]{\textbf{\boldmath{#1}}}
\newcommand{\ie}{{\it i.e.}~}
\newcommand{\eg}{{\it e.g.}~}
\newtheorem{definition}{Definition}

\newtheorem{theorem}{Theorem}

\newtheorem{lemma}{Lemma}

\definecolor{darkgreen}{RGB}{0,125,0}
\newcounter{mlNoteCounter}
\newcounter{vzNoteCounter}

\title{Actor-Critic Policy Optimization in Partially Observable Multiagent  Environments}

\author{Sriram Srinivasan$^{*,1}$\\{\tt srsrinivasan@} \And Marc Lanctot$^{*,1}$\\{\tt lanctot@} \And Vinicius Zambaldi$^1$\\{\tt vzambaldi@} \And Julien P\'{e}rolat$^1$\\{\tt perolat@} \AND
Karl Tuyls$^1$\\{\tt karltuyls@} \And R\'{e}mi Munos$^1$\\{\tt munos@} \And Michael Bowling$^1$\\{\tt bowlingm@}\\
\vspace{0.1cm}\\
\hspace{-8cm}{\tt ...@google.com}. $^1$DeepMind. $^*$These authors contributed equally.
}

\begin{document}

\maketitle

\begin{abstract}
Optimization of parameterized policies for reinforcement learning (RL) is an important
and challenging problem in artificial intelligence. Among the most common approaches are
algorithms based on gradient ascent of a score function representing discounted return.
In this paper, we examine the role of these policy gradient and actor-critic algorithms in partially-observable multiagent environments. We show several candidate policy update rules and relate them to a foundation of regret minimization and multiagent learning techniques for the one-shot and tabular cases. We apply our method to {\it model-free} multiagent reinforcement learning in adversarial sequential decision problems (zero-sum imperfect information games),
using RL-style function approximation.
We evaluate on commonly used benchmark Poker domains, showing performance against fixed policies and empirical convergence to approximate Nash equilibria in self-play with rates similar to or better than a baseline model-free algorithm for zero-sum games, without any domain-specific state space reductions.
\end{abstract}

\section{Introduction}



There has been much success in learning parameterized policies for sequential
decision-making problems. One paradigm driving progress is deep reinforcement learning (Deep RL),
which uses deep learning~\cite{LeCun15} to train function approximators that represent policies,
reward estimates, or both, to learn directly from experience and rewards~\cite{Sutton18}.
These techniques have learned to play Atari games beyond human-level~\cite{Mnih15DQN},
Go, chess, and shogi from scratch~\cite{Silver17AGZ,Silver17AChess}, complex behaviors in
3D environments~\cite{Mnih2016asynchronous,WuTian17,Jaderberg17UNREAL}, robotics~\cite{Gu16, Quillen18}, character animation~\cite{Peng18},
among others.

When multiple agents learn simultaneously, policy optimization becomes more complex.
First, each agent's environment is non-stationary and
naive approaches can be non-Markovian~\cite{Matignon12Independent}, violating the requirements of many
traditional RL algorithms. Second, the optimization problem is not as clearly defined as
maximizing one's own expected reward, because each agent's policy affects the others' optimization problems.
Consequently, game-theoretic formalisms are often used as the basis for representing interactions and
decision-making in multiagent systems~\cite{Busoniu08Comprehensive,Shoham09,Nowe12Game}.

Computer poker is a common multiagent benchmark domain. The presence of partial
observability poses a challenge for traditional RL techniques that exploit the Markov property.
Nonetheless, there has been steady progress in poker AI.
Near-optimal solutions for heads-up limit
Texas Hold'em were found with tabular methods using state aggregation, powered by policy
iteration algorithms based on
regret minimization~\cite{CFR,Tammelin15CFRPlus,Bowling15Poker}. These approaches were founded on a basis of counterfactual regret minimization (CFR), which is the root of recent advances in no-limit, such as Libratus~\cite{Brown17Libratus} and DeepStack~\cite{Moravcik17DeepStack}.
However, (i) both required Poker-specific domain knowledge, and (ii) neither were model-free, and hence are
unable to learn directly from experience, without look-ahead search using a perfect model of the environment.

In this paper, we study the problem of multiagent reinforcement learning in adversarial games with
partial observability, with a focus on the model-free case where agents (a) do not have a perfect description
of their environment (and hence cannot do a priori planning), (b) learn purely from their own experience without
explicitly modeling the environment or other players. We show that actor-critics are related to regret minimization and propose several policy update rules inspired by this connection. We then analyze the convergence properties and present experimental results.

\section{Background and Related Work}

In this section, we briefly describe the necessary background. While we draw on game-theoretic formalisms,
we choose to align our terminology with the RL literature to emphasize the setting and motivations. We include
clarifications in Appendix~\ref{sec:terminology}. For details, see~\cite{Shoham09,Sutton18}.

\subsection{Reinforcement Learning and Policy Gradient Algorithms}

An agent acts by taking actions $a \in \cA$ in states $s \in \cS$ from their policy $\pi : s \rightarrow \Delta(\cA)$, where $\Delta(X)$ is the set of probability distributions over $X$, which results in changing the state of the
environment $s_{t+1} \sim \cT(s_t, a_t)$; the agent then receives an observation $o(s_t, a_t, s_{t+1}) \in \Omega$ and reward $R_t$.\footnote{Note that in fully-observable settings, $o(s_t, a_t, s_{t+1}) = s_{t+1}$. In partially observable environments~\cite{Kaelbling98POMDPs,Oliehoek16}, an observation function $\cO : \cS \times \cA \rightarrow \Delta ( \Omega )$ is used to sample $o(s_t, a_t, s_{t+1}) \sim O(s_t, a_t)$.} 
A sum of rewards is a \defword{return} $G_t = \sum_{t'=t}^{\infty}R_{t'}$, and aim to find $\pi^*$ that maximizes expected return $\bE_\pi[G_0]$.\footnote{
We assume finite episodic tasks of bounded length and leave out the
discount factor $\gamma$ to simplify the notation, without loss of generality.
We use $\gamma( = 0.99)$-discounted returns in our experiments.}

Value-based solution methods achieve this by computing estimates of $v_\pi(s) = \bE_\pi[G_t~|~S_t = s]$, or
$q_\pi(s, a) = \bE_\pi[G_t~|~S_t = s, A_t = a]$,
using temporal difference learning to bootstrap from other estimates, and produce a series of $\epsilon$-greedy policies
$\pi(s,a) = \epsilon / |\cA| + (1-\epsilon) \bI(a = \argmax_{a'} q_\pi(s,a'))$.
In contrast, policy gradient methods define a score function $J(\pi_\theta)$ of some parameterized (and differentiable)
policy $\pi_\theta$ with parameters $\theta$, and use gradient ascent directly on $J(\pi_\theta)$ to update $\theta$.

There have been several recent successful applications of policy gradient algorithms in complex domains such as self-play learning in AlphaGo~\cite{Silver16Go}, Atari and 3D maze navigation~\cite{Mnih2016asynchronous}, continuous control problems~\cite{Schulman15TRPO,Lillicrap16DDPG,Duan16}, robotics~\cite{Gu16}, and autonomous driving~\cite{ShalevShwartz16}.
At the core of several recent state-of-the-art Deep RL algorithms~\cite{Jaderberg17UNREAL,Espeholt18IMPALA} is the advantage actor-critic (A2C) algorithm defined
in~\cite{Mnih2016asynchronous}. In addition to learning a policy ({\it actor}), A2C learns a
parameterized {\it critic}: an estimate of $v_\pi(s)$, which it then uses both to estimate the remaining
return after $k$ steps, and as a control variate (\ie baseline) that reduces the variance of the return estimates.


\subsection{Game Theory, Regret Minimization, and Multiagent Reinforcement Learning \label{sec:gt-rm-marl}}


In multiagent RL (MARL), $n = |\cN| = |\{1, 2, \cdots, n\}|$ agents interact within the same environment.
At each step, each agent $i$ takes an action, and the joint action $\ba$ leads to a new state
$s_{t+1} \sim \cT(s_t, \ba_t$); each player $i$ receives their own separate observation
$o_i(s_t, \ba, s_{t+1})$ and reward $r_{t,i}$. Each agent maximizes their own
return $G_{t,i}$, or their expected return which depends on the joint policy $\bpi$.

Much work in classical MARL focuses on Markov games where the environment is fully
observable and agents take actions simultaneously,
which in some cases admit Bellman operators~\cite{Littman94markovgames,Zinkevich05,Perolat15,Perolat16}. When the
environment is partially observable, policies generally map to values and actions from agents' observation
histories; even when the problem is cooperative, learning is hard~\cite{Oliehoek16}.

We focus our attention to the setting of zero-sum games, where $\sum_{i \in \cN} r_{t,i} = 0$. In this case,
polynomial algorithms exist for finding optimal policies in finite tasks for the two-player case. The guarantees
that Nash equilibrium provides are less clear for the $(n > 2)$-player case, and finding one is hard~\cite{Daskalakis06}.
Despite this, regret minimization approaches are known to filter out dominated actions, and have empirically
found good (\eg competition-winning) strategies in this setting~\cite{Risk10,Gibson13,Lanctot14Further}.

The partially observable setting in multiagent reinforcement learning requires a few more key definitions in order to properly describe the notion of state.
A \defword{history} $h \in \cH$ is a sequence of actions from all players {\it including the environment} taken from the start of an episode.
The environment (also called ``nature'') is treated as a player with a fixed policy, such that there is a deterministic mapping from any $h$ to the actual state of the environment. 
Define an \defword{information state}, $s_t = \{ h \in \cH~|~$ player $i$'s sequence of observations,
$o_{i,t'<t}(s_{t'}, \ba_{t'}, s_{t'+1})$, is consistent with $h$\}\footnote{In defining $s_t$,
we drop the reference to acting player $i$ in turn-based games without loss of generality.}. So, $s_t$ includes histories
leading to $s_t$ that are indistinguishable to player $i$;
\eg in Poker, the $h \in s_t$ differ only in the private cards dealt to opponents.
A joint policy $\bpi$ is a \defword{Nash equilibrium} if the incentive to deviate to a best response
$\delta_i(\bpi) = \max_{\pi_i'} \bE_{\pi_i', \pi_{-i}}[G_{0, i}] - \bE_{\bpi}[G_{0, i}] = 0$ for each player $i \in \cN$,
where $\pi_{-i}$ is the set of $i's$ opponents' policies.
Otherwise, $\epsilon$-equilibria are approximate, with $\epsilon = \max_i \delta_i(\bpi)$.
Regret minimization algorithms produce iterates whose average policy $\bar{\bpi}$ reduces an upper
bound on $\epsilon$; convergence is measured using
$\textsc{NashConv}(\bpi) = \sum_i \delta_i(\bpi)$. Nash equilibrium is minimax-optimal in
two-player zero-sum games, so using one minimizes worst-case losses.

There are well-known links between learning, game theory and regret minimization~\cite{Blum07}.
One method, counterfactual regret (CFR) minimization~\cite{CFR}, has led to significant progress in Poker AI.
Let $\eta^\bpi(h_t) = \prod_{t' < t} \bpi(s_{t'}, a_{t'})$, where  $h_{t'} \sqsubset h_t$ is a prefix, $h_{t'} \in s_{t'}, h_t \in s_t$, be the \defword{reach probability} of $h$ under $\pi$ from all policies' action choices.
This can be split into player $i$'s contribution and their opponents' (including nature's) contribution,  $\eta^\pi(h) = \eta^\bpi_i(h) \eta^\bpi_{-i}(h)$.
Suppose player $i$ is to play at $s$: under \defword{perfect recall}, player $i$ remembers the sequence of their own states reached, which is the same for all $h \in s$, since they differ only in private information seen by opponent(s); as a result $\forall h, h' \in s, \eta_i^\pi(h) = \eta_i^\pi(h') := \eta_i^{\pi}(s)$.
For some history $h$ and action $a$, we call $h$ a \defword{prefix history} $h \sqsubset ha$, where $ha$ is the history $h$ followed by action $a$; they may also be smaller, so $h \sqsubset ha \sqsubset hab \Rightarrow h \sqsubset hab$. 
Let $\cZ = \{ z \in \cH~|~\mbox{$z$ is terminal} \}$ and $\cZ(s,a) = \{ (h,z) \in \cH \times \cZ~|~h \in s, ha \sqsubseteq z\}$. CFR defines \defword{counterfactual values}
$v^c_i(\bpi, s_t, a_t) = \sum_{(h,z) \in \cZ(s_t, a_t)} \eta^\pi_{-i}(h) \eta^\pi_i(z) u_i(z)$, 
and $v^c_i(\bpi, s_t) = \sum_a \pi(s_t,a) v_i^c(\bpi, s_t, a_t)$, where $u_i(z)$ is the return to player $i$ along $z$, 
and accumulates regrets $\textsc{reg}_i(\bpi, s_t, a') = v_i^c(\bpi, s_t, a') - v_i^c(\bpi, s_t)$,
producing new policies from cumulative regret using \eg
regret-matching~\cite{Hart00} or exponentially-weighted experts~\cite{Exp3,Brown17}.

CFR is a policy iteration algorithm that computes the expected values by visiting every possible
trajectory, described in detail in Appendix~\ref{sec:cfr}. Monte Carlo CFR (MCCFR) samples trajectories using an exploratory behavior policy,
computing unbiased estimates $\hat{v}_i^c(\bpi, s_t)$ and $\widehat{\textsc{reg}}_i(\bpi, s_t)$ corrected by
importance sampling~\cite{Lanctot09mccfr}. Therefore, MCCFR is an
{\it off-policy Monte Carlo} method. In one MCCFR variant, \defword{model-free outcome sampling} (MFOS),
the behavior policy at opponent states is defined as $\pi_{-i}$ enabling online regret minimization
(player $i$ can update their policy independent of $\pi_{-i}$ and $\cT$).

There are two main problems with (MC)CFR methods: (i) significant variance is introduced by sampling (off-policy)
since quantities are divided by reach probabilities, (ii) there is no generalization across states except through expert abstractions and/or forward simulation with a perfect model.
We show that actor-critics address both problems and that they are a form of {\it on-policy} MCCFR.

\subsection{Most Closely Related Work}


There is a rich history of policy gradient approaches in MARL.
Early uses of gradient ascent showed that cyclical learning dynamics could arise, even in zero-sum matrix games~\cite{Singh00}. This was partly addressed by methods
that used variable learning rates~\cite{Bowling02,Bowling04}, policy prediction~\cite{Zhang10},
and weighted updates~\cite{Abdallah08}. The main limitation with these classical
works was scalability: there was no direct way to use function approximation, and empirical analyses
focused almost exclusively on one-shot games.

Recent work on policy gradient approaches to MARL addresses scalability by using newer
algorithms such as A3C or TRPO~\cite{Schulman15TRPO}.
Naive approaches such as independent reinforcement learning fail to find optimal
stochastic policies~\cite{Littman94markovgames,Heinrich16} and can overfit the training
data, failing to generalize during execution~\cite{Lanctot17PSRO}.
Considerable progress
has been achieved for cooperative MARL: learning to communicate~\cite{Lazaridou17}, Starcraft unit
micromanagement~\cite{Foerster17}, taxi fleet optimization~\cite{Nguyen17},
and autonomous driving~\cite{ShalevShwartz16}. There has also been significant progress for mixed
cooperative/competitive environments: using a centralized critic~\cite{Lowe17}, learning to
negotiate~\cite{Kao18}, anticipating/learning opponent responses in
social dilemmas~\cite{Foerster18,Lerer17},
and control in realistic physical environments~\cite{AlShedivat18,Bansal18}.
In this line of research, the most common evaluation methodology has been
to train centrally (for decentralized execution), either having direct access to the other
players' policy parameters or modeling them
explicitly. As a result, assumptions are made about the form of the other agents' policies, utilities,
or learning mechanisms.

There are also methods that attempt to model the opponents~\cite{dpiqn,He16DRON,Albrecht18Modeling}. Our methods do no such modeling, and can be classified in the ``forget'' category of the taxonomy proposed in~\cite{HernandezLeal18Survey}: that is, due to its on-policy nature, actors and critics adapt to and learn mainly from new/current experience.

We focus on the {\it model-free} (and online) setting: other agents' policies are inaccessible; training is not separated from execution.
Actor-critics were recently studied in this setting for multiagent games~\cite{Perolat18}, whereas we focus on partially-observable environments; only tabular methods are known to converge.
Fictitious Self-Play computes approximate best responses via RL~\cite{Heinrich15FSP,Heinrich16}, and can also be model-free.
Regression CFR (RCFR) uses regression to estimate cumulative regrets from CFR~\cite{Waugh15solving}.
RCFR is closely related to Advantage Regret Minimization (ARM)~\cite{Jin17ARM}.
ARM~\cite{Jin17ARM} shows regret estimation methods handle partial observability better than standard RL, but was not evaluated in multiagent environments. 
In contrast, we focus primarily on the multiagent setting.

\section{Multiagent Actor-Critics: Advantages and Regrets  \label{sec:rpg}}



CFR defines policy update rules from thresholded cumulative counterfactual regret: $\textsc{tcreg}_i(K, s,a) = ( \sum_{k \in \{1, \cdots, K\}} \textsc{reg}_i(\pi_k, s, a) )^+$, where $k$ is the number of iterations and $(x)^+ = \max(0, x)$. 
In CFR, regret matching updates a policy to be proportional to $\textsc{tcreg}_i(K,s,a)$. 

On the other hand, REINFORCE~\cite{Williams92} samples trajectories and computes gradients for each state $s_t$, updating $\btheta$ toward $\nabla_{\btheta} \log(s_t, a_t; \btheta) G_t$.
A baseline is often subtracted from the return: $G_t - v_\pi(s_t)$, and policy gradients then become actor-critics, training $\pi$ and $v_\pi$ separately.
The log appears due to the fact that action $a_t$ is sampled from the policy, the value is divided by $\pi(s_t, a_t)$ to ensure the estimate is properly estimating the true expectation~\cite[Section 13.3]{Sutton18}, and $\nabla_{\btheta} \pi(s_t, a_t; \btheta) / \pi(s_t, a_t, \btheta) = \nabla_{\btheta} \log \pi(s_t, a_t; \btheta)$. One could instead train $q_\pi$-based critics from states {\it and} actions. This leads to a $q$-based Policy Gradient (QPG) (also known as Mean Actor-Critic~\cite{Allen18MAC}):
\begin{equation}
    \nabla_{\btheta}^{\textsc{QPG}}(s) = \sum_a [\nabla_\theta \pi(s, a; \btheta)] \left(q(s,a; \bw) - \sum_b \pi(s, b; \btheta)  q(s,  b, \bw)\right),
\label{eq:qac}
\end{equation}
an advantage actor-critic algorithm differing from A2C in the (state-action) representation of the critics~\cite{liu2018action,Wu18} and summing over actions similar to the all-action algorithms~\cite{Sutton01Comparing,Peters02Policy,Ciosek18EPG,Allen18MAC}.
Interpreting $a_\pi(s,a) = q_\pi(s,a) - \sum_b \pi(s,b) q_\pi(s,b)$ as a regret, we can instead minimize a loss defined by an upper bound on the thresholded cumulative regret: $\sum_k (a_{\pi_k}(s,a))^+ \ge (\sum_k (a_{\pi_k}(s,a))^+$, moving the policy toward a no-regret region.
We call this Regret Policy Gradient (RPG):
\begin{equation}
    \nabla_{\btheta}^{\textsc{RPG}}(s) = -\sum_a \nabla_\theta \left(q(s,a; \bw) - \sum_b \pi(s, b; \btheta) q(s,b; \bw)\right)^+.
\label{eq:rpg}
\end{equation}
The minus sign on the front represents a switch from gradient ascent on the score to {\it descent} on the loss.
Another way to implement an adaptation of the regret-matching rule is by weighting the policy gradient by the thresholded regret, which we call Regret Matching Policy Gradient (RMPG):
\begin{equation}
    \nabla_{\btheta}^{\textsc{RMPG}}(s) = \sum_a [\nabla_\theta \pi(s, a; \btheta)] \left(q(s,a; \bw) - \sum_b \pi(s, b; \btheta) q(s, b, \bw)\right)^+.
\label{eq:rmpg}
\end{equation}

In each case, the critic $q(s_t, a_t; \bw)$ is trained in the standard way, using $\ell_2$ regression loss from sampled returns.
The pseudo-code is given in Algorithm~\ref{alg:rpg} in Appendix~\ref{sec:pseudocode}.
In Appendix~\ref{sec:app-qpg-rpg}, we show that the QPG gradient is proportional to the RPG gradient at $s$: $\nabla_{\btheta}^\textsc{RPG}(s) \propto \nabla_{\btheta}^\textsc{QPG}(s)$.

\subsection{Analysis of Learning Dynamics on Normal-Form Games}


The first question is whether any of these variants can converge to an equilibrium, even in the simplest case. 
So, we now show phase portraits of the learning dynamics on Matching Pennies: a two-action version of Rock, Paper, Scissors.
These analyses are common in multiagent learning as they allow visual depiction of the policy changes and how different factors affect the (convergence) behavior~\cite{Singh00,walsh:02,Bowling02,Walsh03,Bowling04,Wellman06,Abdallah08,Zhang10,Wunder2010,BloembergenTHK15,Tuyls18}.
Convergence is difficult in Matching Pennies as the only Nash equilibrium $\pi^* = ((\frac{1}{2}, \frac{1}{2}), (\frac{1}{2}, \frac{1}{2}))$ requires learning stochastic policies.
We give more detail and results on different games that cause cyclic learning behavior in Appendix~\ref{sec:app-dynamics}.

In Figure~\ref{fig:mp-dynamics}, we see the similarity of the regret dynamics to replicator dynamics~\cite{TaylorJonkerRD,Sandholm17}.
We also show the {\it average policy
dynamics} and observe convergence to equilibrium in each game we tried, which is a known to be guaranteed in two-player zero-sum games using
CFR, fictitious play~\cite{Brown51}, and continuous replicator dynamics~\cite{Hofbauer09}. However, computing the average policy is complex~\cite{Heinrich15FSP,CFR} and potentially worse with function approximation, requiring storing past data in large buffers~\cite{Heinrich16}. 

\begin{figure}[t!]
\begin{tabular}{ccc}
\hspace{-0.3cm}
\includegraphics[width=0.32\textwidth]{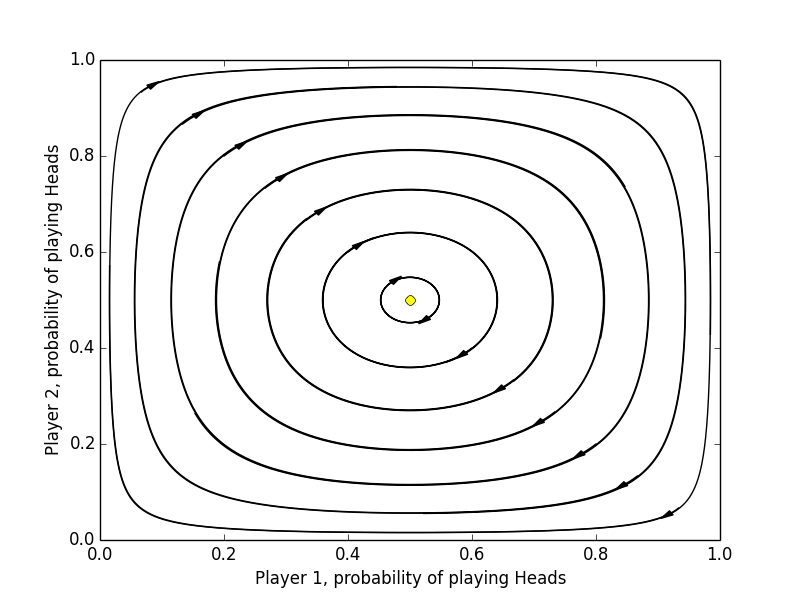} &
\includegraphics[width=0.32\textwidth]{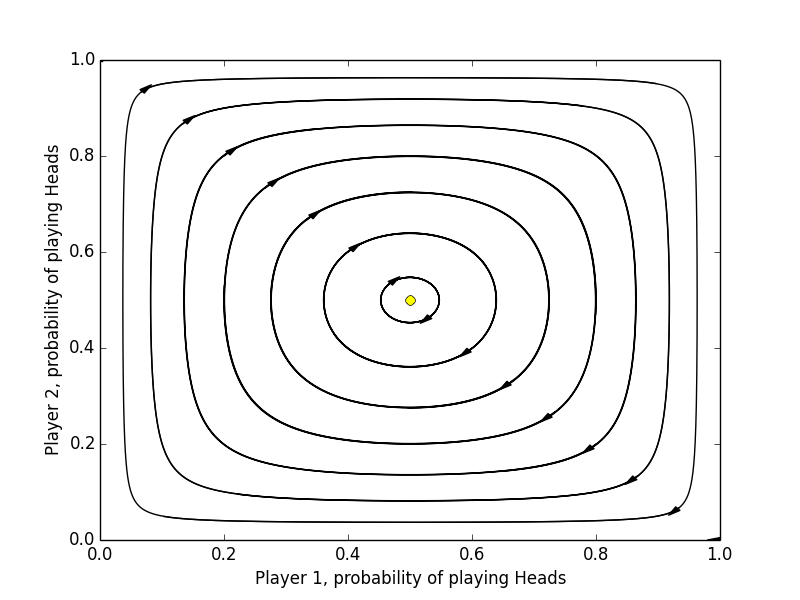} & 
\includegraphics[width=0.32\textwidth]{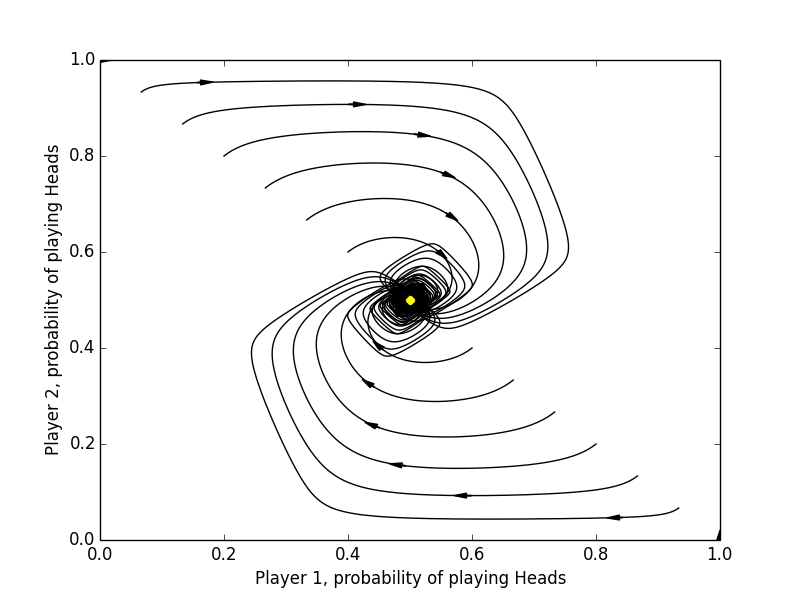} \\
(a) Replicator Dynamics & (b) RPG Dynamics & (c) Average RPG Dynamics \\
\end{tabular}
\caption{Learning Dynamics in Matching Pennies: (a) and (b) show the vector field for ${\partial \pi} / {\partial t}$ including example particle traces, where each point is each player's probability of their first action; (c) shows example traces of policies following a discrete approximation to $\int_0^t {\partial \pi}/{\partial t}$. \label{fig:mp-dynamics}} 
\end{figure}


\subsection{Partially Observable Sequential Games \label{sec:rpg-seq}}


How do the values $v_i^c(\bpi, s_t, a_t)$ and $q_{\bpi,i}(s_t, a_t)$ differ?
The authors of \cite{Jin17ARM} posit that they are approximately equal when $s_t$ rarely occurs more than once in a trajectory.
First, note that $s_t$ cannot be reached more than once in a trajectory from our definition of $s_t$, because the observation histories (of the player to play at $s_t$) would be different in each occurrence (\ie due to perfect recall).
So, the two values are indeed equal in deterministic, single-agent environments.
In general, counterfactual values are conditioned on {\it player $i$ playing to reach $s_t$}, whereas $q$-function estimates are conditioned on {\it having reached $s_t$}. So, $q_{\pi,i}(s_t,a_t) = \bE_{\rho \sim \bpi}[ G_{t,i}~|~S_t = s_t, A_t = a_t ]$
\begin{eqnarray*}
                & = & \sum_{h,z \in \cZ(s_t,a_t)} \Pr(h~|~s_t) \eta^\pi(ha,z) u_i(z)~~~~~~~~~~~~~~~~~~~~~~~~~\text{where $\eta^\pi(ha,z) = \frac{\eta^\pi(z)}{\eta^\pi(h) \pi(s,a)}$} \\
                & = & \sum_{h,z \in \cZ(s_t,a_t)} \frac{\Pr(s_t~|~h)  \Pr(h)}{\Pr(s_t)} \eta^\pi(ha,z) u_i(z) ~~~~~~~~~~~~~~\text{by Bayes' rule}\\ 
                & = & \sum_{h,z \in \cZ(s_t,a_t)} \frac{\Pr(h)}{\Pr(s_t)} \eta^\pi(ha,z) u_i(z)~~~~~~~~~~~~~~~~~~~~~~~~~~~~~~\text{since $h \in s_t$, $h$ is unique to $s_t$}\\
                & = & \sum_{h,z \in \cZ(s_t,a_t)} \frac{\eta^\bpi(h)}{\sum_{h' \in s_t}{\eta^\bpi(h')}} \eta^\pi(ha,z) u_i(z) \\
                & = & \sum_{h,z \in \cZ(s_t,a_t)} \frac{\eta_i^\bpi(h) \eta_{-i}^\bpi(h)}{\sum_{h' \in s_t}{\eta_i^\bpi(h') \eta_{-i}^\bpi(h')}} \eta^\pi(ha,z) u_i(z) \\
                & = & \sum_{h,z \in \cZ(s_t,a_t)} \frac{\eta_i^\bpi(s) \eta_{-i}^\bpi(h)}{\eta_i^\bpi(s) \sum_{h' \in s_t}{\eta_{-i}^\bpi(h')}} \eta^\pi(ha,z) u_i(z)~~~~~\text{due to def. of $s_t$ and perfect recall}\\
                & = & \sum_{h,z \in \cZ(s_t,a_t)} \frac{\eta_{-i}^\bpi(h)}{\sum_{h' \in s_t}{\eta_{-i}^\bpi(h')}} \eta^\pi(ha,z) u_i(z)~~=~~\frac{1}{\sum_{h \in s_t}\eta^\bpi_{-i}(h)} v_i^c(\pi, s_t, a_t).\\
\end{eqnarray*}

The derivation is similar to show that $v_{\bpi,i}(s_t) = v_i^c(\pi, s_t) / \sum_{h \in s_t} \eta^\bpi_{-i}(h)$.
Hence, counterfactual values and standard value functions are generally not equal, but are scaled by the Bayes normalizing constant $\cB_{-i}(\bpi, s_t) = \sum_{h \in s_t} \eta^\bpi_{-i}(h)$. If there is a low probability of reaching $s_t$ due to the environment or due to opponents' policies, these values will differ significantly.

This leads to a new interpretation of actor-critic algorithms in the multiagent partially observable setting: the advantage values $q_{\bpi,i}(s_t, a_t) - v_{\bpi,i}(s_t, a_t)$ are immediate counterfactual regrets scaled by $1/\cB_{-i}(\bpi, s_t)$. 

Note that the standard policy gradient theorem holds: gradients can be estimated from samples. 
This follows from the derivation of the policy gradient in the tabular case (see Appendix~\ref{sec:app-sequential}).
When TD bootstrapping is not used, the Markov property is not required; having multiple agents and/or partial observability does not change this. For a proof using REINFORCE ($G_t$ only), see~\cite[Theorem 1]{ShalevShwartz16}. The proof trivially follows using $G_{t,i} - v_{\bpi,i}$ since $v_{\bpi,i}$ is trained separately and does not depend on $\rho$. 

Policy gradient algorithms perform gradient ascent on $J^{PG}(\pi_{\btheta}) = v_{\bpi_\theta}(s_0)$, using $\nabla_{\btheta} J^{PG}(\bpi_{\btheta}) \propto \sum_s \mu(s) \sum_a \nabla_{\btheta} \pi_{\theta}(s,a) q_{\pi}(s,a)$, where $\mu$ is on-policy distribution under $\pi$~\cite[Section 13.2]{Sutton18}.
The actor-critic equivalent is $\nabla_{\btheta} J^{AC}(\pi_{\btheta}) \propto \sum_s \mu(s) \sum_a \nabla_{\btheta} \pi_{\theta}(s,a) (q_{\pi}(s,a) - \sum_b \pi(s,b) q_{\pi}(s,b))$. Note that the baseline is unnecessary when summing over the actions and $\nabla_{\btheta} J^{AC}(\pi_{\btheta}) = \nabla_{\btheta} J^{PG}(\pi_{\btheta})$~\cite{Allen18MAC}. 
However, our analysis relies on a projected gradient descent algorithm that does not assume simplex constraints on the policy: in that case, in general $\nabla_{\btheta} J^{AC}(\pi_{\btheta}) \not= \nabla_{\btheta} J^{PG}(\pi_{\btheta})$.
\begin{definition}
\label{def:acpi}
Define \defword{policy gradient policy iteration} (PGPI) as a process that iteratively runs $\btheta \leftarrow \btheta + \alpha \nabla_{\btheta} J^{PG}(\pi_{\btheta})$, and 
\defword{actor-critic policy iteration} (ACPI) similarly using $\nabla_{\btheta} J^{AC}(\pi_{\btheta})$.
\end{definition}
In two-player zero-sum games, PGPI/ACPI are gradient ascent-descent problems, because each player is trying to ascend their own score function, and when using tabular policies a solution exists due to the minimax theorem~\cite{Shoham09}. 
Define player $i$'s \defword{external regret} over $K$ steps as $R_i^K = \max_{\pi_i'\in \Pi_i} \left( \sum_{k=1}^K  \bE_{\pi_i'}[G_{0,i}] - \bE_{\pi^k}[G_{0,i}] \right) $, where $\Pi_i$ is the set of deterministic policies.
\vspace{0.1cm}
\begin{theorem}
\label{thm:pg-cfr-conv}
In two-player zero-sum games, when using tabular policies and an $\ell_2$ projection $P(\btheta) =  \argmin_{\btheta' \in \Delta(\cS, \cA)} \Vert \btheta - \btheta' \Vert_2$, where $\Delta(\cS, \cA) = \{ \btheta~|~ \forall s \in \cS, \sum_{b \in \cA} \btheta_{s,b} = 1\}$ is the space of tabular simplices, if player $i$ uses learning rates of $\alpha_{s,k} =  k^{-\frac{1}{2}}$ at $s$ on iteration $k$, and $\theta^k_{s,a} > 0$ for all $k$ and $s$, then projected PGPI, $\theta^{k+1}_{s, \cdot} \leftarrow P( \{ \theta^k_{s,a} + \alpha_{s,k} \frac{\partial}{\partial \theta^k_{s,a}} J^{PG}(\pi_{\btheta^k}) \}_a )$, has regret $R_i^K \le \frac{1}{\eta_i^{\min}} |\cS_i| \left( \sqrt{K} + (\sqrt{K} - \frac{1}{2}) |\cA| (\Delta r)^2 \right) + O(K)$, where $\cS_i$ is the set of player $i$'s states, $\Delta r$ is the reward range, and $\eta_i^{\min} = \min_{s,k} \eta_i^k(s)$. The same holds for projected ACPI (see appendix).
\end{theorem}

The proof\footnote{Note that the theorem statements and proofs differ from their original form. See Appendix~\ref{app:errata} for details.} is given in Appendix~\ref{sec:app-sequential}. 
In the case of sampled trajectories, as long as every state is reached with positive probability, Monte Carlo estimators of $q_{\bpi,i}$ will be consistent.
Therefore, we use exploratory policies and decay exploration over time. With a finite number of samples, the probability that an estimator $\hat{q}_{\bpi,i}(s,a)$ differs by some quantity away from its mean is determined by Hoeffding's inequality and the reach probabilities. We suspect these errors could be accumulated to derive probabilistic regret bounds similar to the off-policy Monte Carlo case~\cite{Lanctot09Sampling}.

The analysis so far has concentrated on establishing relationships for the optimization problem that underlies standard formulation of policy gradient and actor-critic algorithms. A different bound can be achieved by using stronger policy improvement (proof and details are found in Appendix~\ref{sec:app-sequential}):

\begin{theorem}
\label{thm:strong-acpi}
Define a state-local $J^{PG}(\bpi_{\btheta}, s) = v_{\pi_{\btheta},i}(s)$, composite gradient $\{ \frac{\partial}{\partial \theta_{s,a}} J^{PG}(\bpi_{\btheta}, s) \}_{s,a}$, \defword{strong policy gradient policy iteration} (SPGPI), and \defword{strong actor-critic policy iteration} (SACPI) as in Definition~\ref{def:acpi} except replacing the gradient components with $\frac{\partial}{\partial \theta_{s,a}} J^{PG}(\pi_{\btheta},s)$. Then, in two-player zero-sum games, when using tabular policies and projection $P(\btheta)$ as defined in Theorem~\ref{thm:pg-cfr-conv} with learning rates $\alpha_{k} = k^{-\frac{1}{2}}$ on iteration $k$, projected SPGPI, $\theta^{k+1}_{s, \cdot} \leftarrow P( \{ \theta^k_{s,a} + \alpha_k \frac{\partial}{\partial \theta^k_{s,a}} J^{PG}(\bpi_{\btheta}, s) \}_a)$, has regret $R_i^K \le |\cS_i| \left( \sqrt{K} + (\sqrt{K} - \frac{1}{2}) |\cA| (\Delta r)^2 \right) + O(K)$, where $\cS_i$ is the set of player $i$'s states and $\Delta r$ is the reward range. This also holds for projected SACPI (see appendix).
\end{theorem}

Note that none of the variants are guaranteed to have regret sublinear in $K$, so they may not converge to a Nash equilibrium. We discuss this further in the appendix.

\section{Empirical Evaluation}

We now assess the behavior of the actor-critic algorithms in practice.
While the analyses in the previous section established relationships for the tabular case, ultimately we want to assess scalability and generalization potential for larger settings. Our implementation parameterizes critics and policies using neural networks with two fully-connected layers of 128 units each, and rectified linear unit activation functions, followed by a linear layer to output a single value $q$ or softmax layer to output $\pi$. We chose these architectures to remain consistent with previous evaluations~\cite{Heinrich16,Lanctot17PSRO}. 

\subsection{Domains: Kuhn and Leduc Poker}

We evaluate the actor-critic algorithms on two $n$-player games: Kuhn poker, and Leduc poker.

\defword{Kuhn poker} is a toy game where each player starts with 2 chips, antes 1 chip to play, and receives one card face down from a deck of size $n+1$ (one card remains hidden). Players proceed by betting (raise/call) by adding their remaining chip to the pot, or passing (check/fold) until all players are either in (contributed as all other players to the pot) or out (folded, passed after a raise). The player with the highest-ranked card that has not folded wins the pot.

In \defword{Leduc poker}, players have a limitless number of chips, and the deck has size $2(n+1)$, divided into two suits of identically-ranked cards. There are two rounds of betting, and after the first round a single public card is revealed from the deck. Each player antes 1 chip to play, and the bets are limited to two per round, and number of chips limited to 2 in the first round, and 4 in the second round.

The rewards to each player is the number of chips they had after the game minus before the game. To remain consistent with other baselines, we use the form of Leduc described in~\cite{Lanctot17PSRO} which does not restrict the action space, adding reward penalties if/when illegal moves are chosen.

\subsection{Baseline: Neural Fictitious Self-Play}

We compare to one main baseline. \defword{Neural Fictitious Self-Play} (NFSP) is an implementation of fictitious play, where approximate best responses are used in place of full best response~\cite{Heinrich16}. Two transition buffers of are used: $\cD^{RL}$ and $\cD^{ML}$; the former to train a DQN agent towards a best response $\pi_i$ to $\bar{\pi}_{-i}$, data in the latter is replaced using reservoir sampling, and trains $\bar{\bpi}_i$ by classification.


\subsection{Main Performance Results}

Here we show the empirical convergence to approximate Nash equlibria for each algorithm in self-play, and performance against fixed bots. The standard metric to use for this is \textsc{NashConv}($\bpi$) defined in Section~\ref{sec:gt-rm-marl}, which reports the accuracy of the approximation to a Nash equilibrium.

{\bf Training Setup}.
In the domains we tested, we observed that the variance in returns was high and hence we performed multiple policy evaluation updates ($q$-update for $\nabla^\textsc{QPG}$
, $\nabla^\textsc{RPG}$, and $\nabla^\textsc{RMPG}$, and $v$-update for A2C) followed by policy improvement (policy gradient update). These updates were done using separate SGD optimizers with their respective learning rates of fixed 0.001 for policy evaluation, and annealed from a starting learning rate to 0 over 20M steps for policy improvement. (See Appendix \ref{sec:app-experiments} for exact values). Further, the policy improvement step is applied after $N_q$ policy evaluation updates. We treat $N_q$ and batch size as a hyper parameters and sweep over a few reasonable values.
In order to handle different scales of rewards in the multiple domains, we used the streaming Z-normalization on the rewards, inspired by its use in Proximal Policy Optimization (PPO)~\cite{schulman2017proximal}.
In addition, the agent's policy is controlled by a(n inverse) temperature added as part of the softmax operator. The temperature is annealed from 1 to 0 over 1M steps to ensure adequate state space coverage. An additional entropy cost hyper-parameter is added as is standard practice with Deep RL policy gradient methods such as A3C~\cite{Mnih2016asynchronous,schulman2017proximal}.
For NFSP, we used the same values presented in~\cite{Lanctot17PSRO}.


{\bf Convergence to Equilibrium.}
See Figure~\ref{fig:conv} for convergence results. Please note that we plot the \textsc{NashConv} for the average policy in the case of NFSP, and the current policy in the case of the policy gradient algorithms. We see that in 2-player Leduc, the actor-critic variants we tried are similar in performance; NFSP has faster short-term convergence but long-term the actor critics are comparable. Each converges significantly faster than A2C. 
However RMPG seems to plateau.

{\bf Performance Against Fixed Bots.}
We also measure the expected reward against  fixed bots, averaged over player seats. These bots, \textsc{cfr500}, correspond to the average policy after 500 iterations of CFR. QPG and RPG do well here, scoring higher than A2C and even beating NFSP in the long-term. 

\begin{figure}[t]
\centering
\begin{tabular}{cc}
\includegraphics[width=0.46\textwidth]{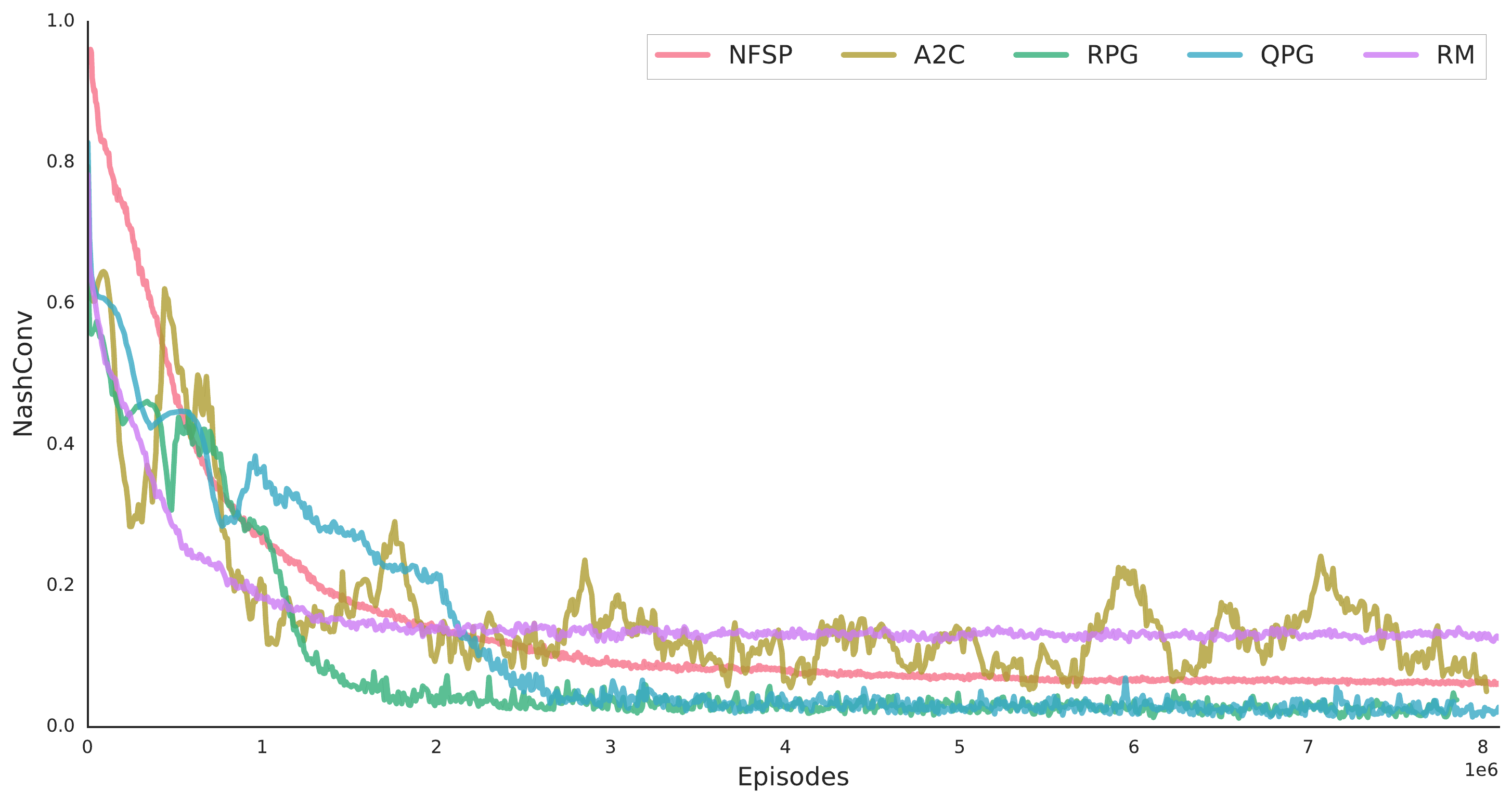} &
\includegraphics[width=0.46\textwidth]{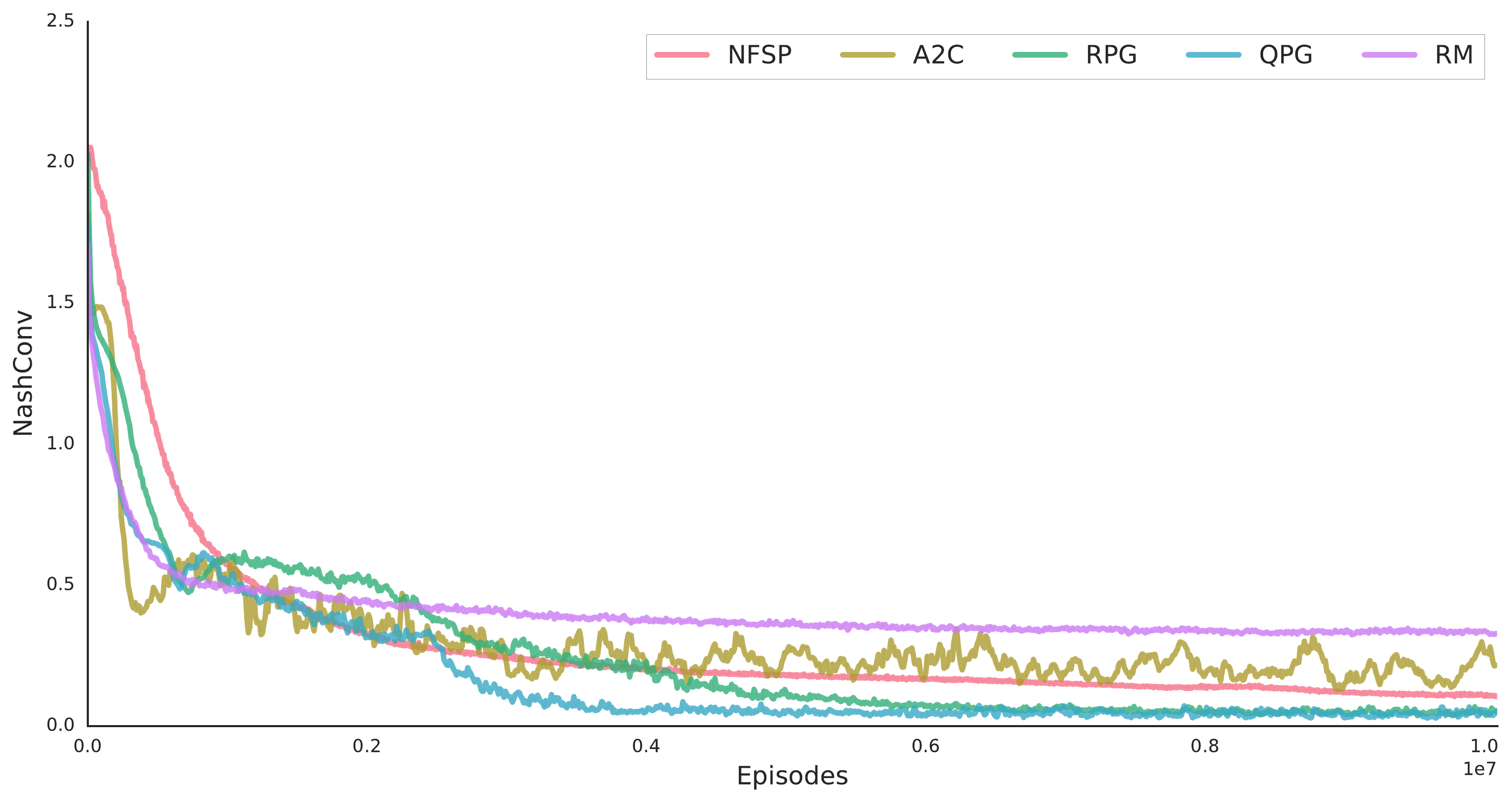} \\
\textsc{NashConv} in 2-player Kuhn  & \textsc{NashConv} in 3-player Kuhn \\
\includegraphics[width=0.46\textwidth]{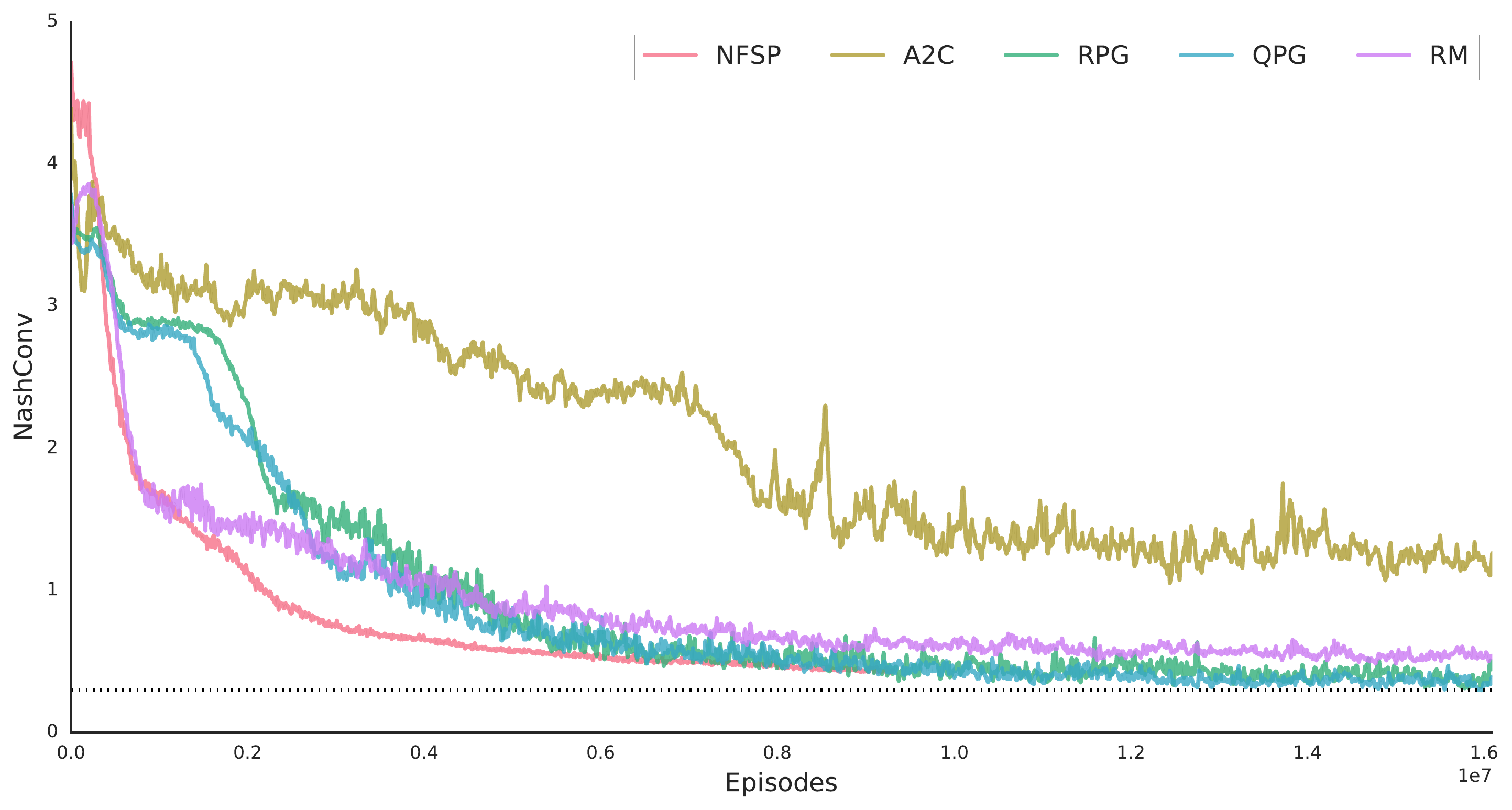} &
\includegraphics[width=0.46\textwidth]{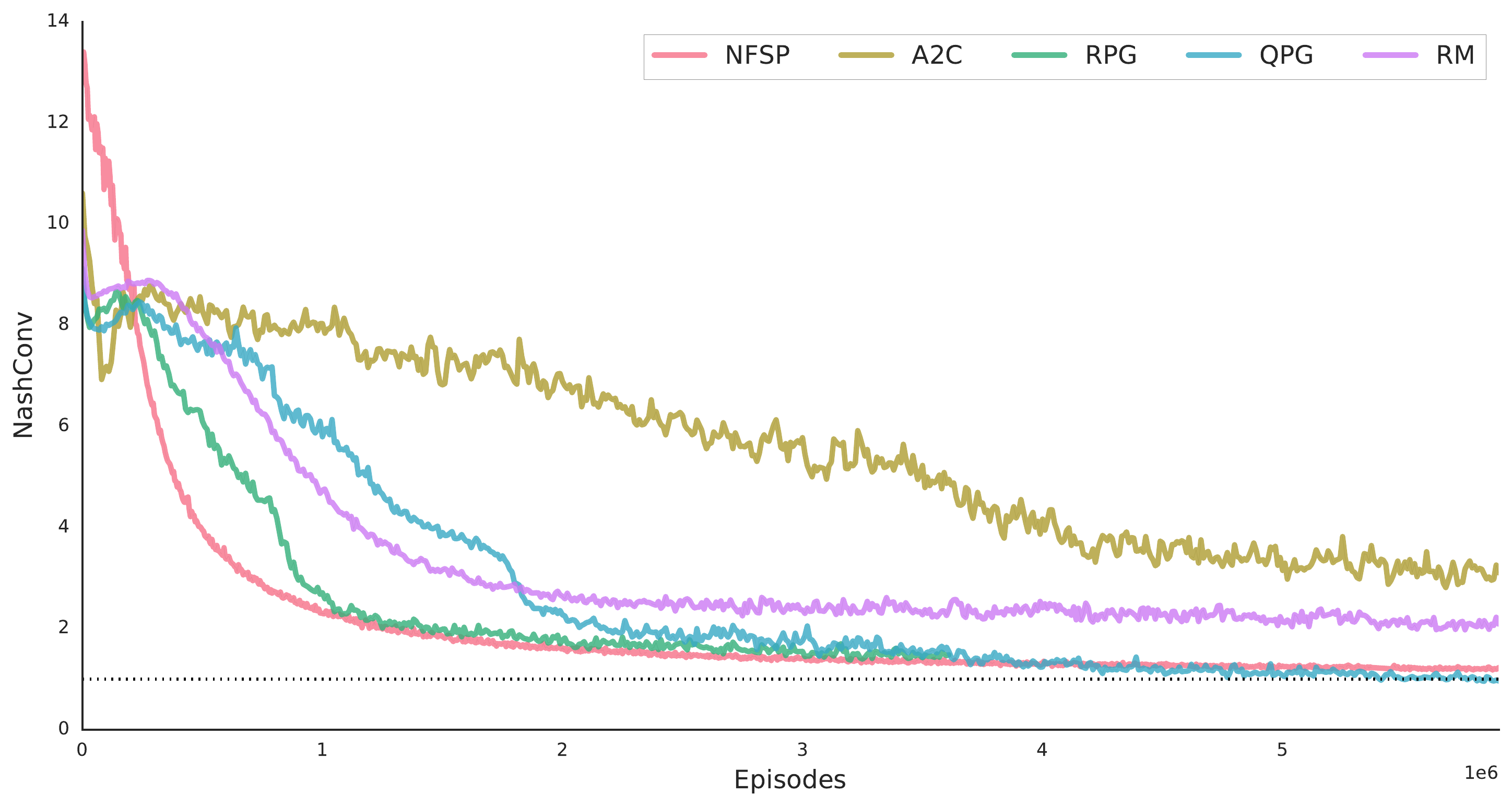} \\
\textsc{NashConv} in 2-player Leduc & \textsc{NashConv} in 3-player Leduc \\
 \includegraphics[width=0.46\textwidth]{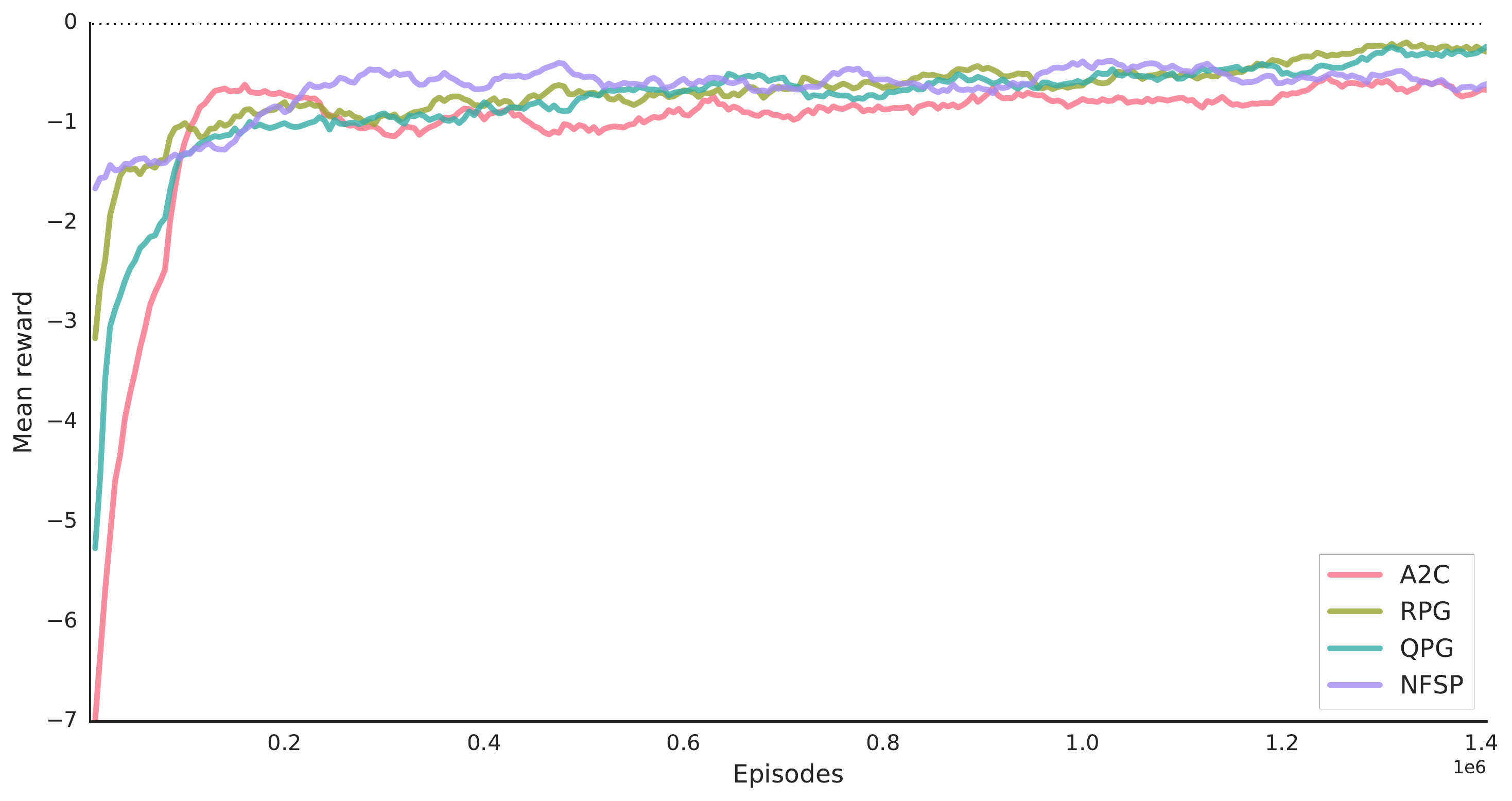} & 
 \includegraphics[width=0.46\textwidth]{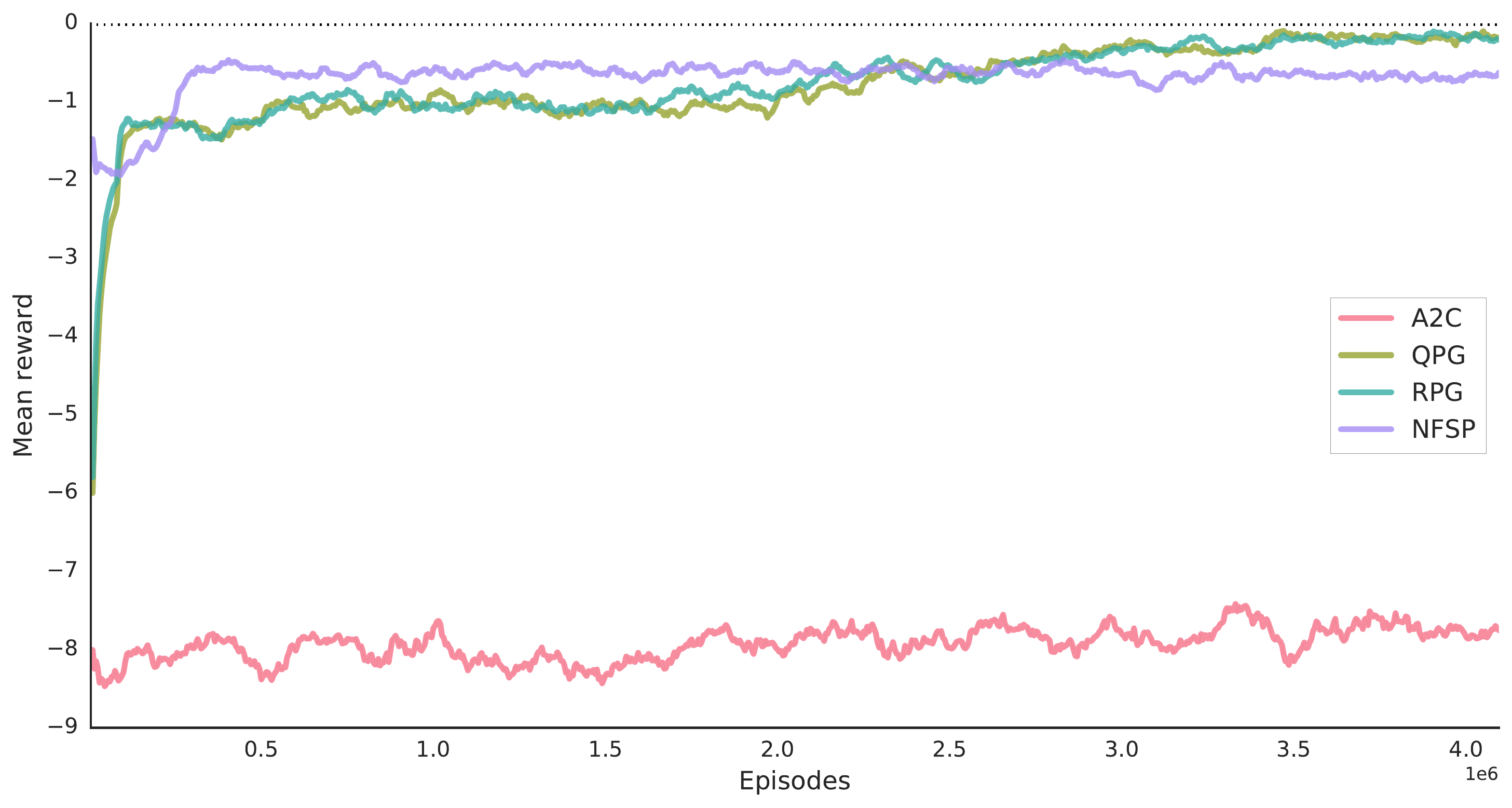} \\
 2-player Leduc vs. \textsc{CFR500} & 3-player Leduc vs \textsc{CFR500} \\
\end{tabular}  
\caption{Empirical convergence rates for \textsc{NashConv}($\bpi$) and performance versus CFR agents. \label{fig:conv}}
\end{figure}


\section{Conclusion}



In this paper, we discuss several update rules for actor-critic algorithms in multiagent reinforcement learning. One key property of this class of algorithms is that they are model-free, leading to a purely online algorithm, independent of the opponents and environment.
We show a connection between these algorithms and (counterfactual) regret minimization.

Our experiments show that these actor-critic algorithms converge to approximate Nash equilibria in commonly-used benchmark Poker domains with rates similar to or better than baseline model-free algorithms for zero-sum games. However, they may be easier to implement, and do not require storing a large memory of transitions. Furthermore, the current policy of some variants do significantly better than the baselines (including the average policy of NFSP) when evaluated against fixed bots. 
Of the actor-critic variants, RPG and QPG seem to outperform RMPG in our experiments.

As future work, we would like to formally develop the (probabilistic) guarantees of the sample-based on-policy Monte Carlo CFR algorithms and/or extend to continuing tasks as in MDPs~\cite{kash2019combining}. We are also curious about what role the connections between actor-critic methods and CFR could play in deriving convergence guarantees in model-free MARL for cooperative and/or potential games.
%

{\bf Acknowledgments.} We would like to thank Martin Schmid, Audr\={u}nas Gruslys, Neil Burch, Noam Brown, Kevin Waugh, Rich Sutton, and Thore Graepel for their helpful feedback and support.

\bibliography{rpg}
\bibliographystyle{plain}

\newpage

\appendix
{\Large {\bf Appendices}}

\section{Some Notes on Notation and Terminology \label{sec:terminology}}

Here we clarify some notational differences between the work on computational game theory and (multiagent) reinforcement learning.

There are some analogs between approximate dynamic programming and RL to counterfactual regret minimization in zero-sum games.

CFR is a policy iteration technique: it implements generalized policy iteration: policy evaluation computes the (counterfactual) values $v_\pi^c$. Policy ``improvement'' is implemented by the regret minimizers at each information state, such as regret-matching which yields the next policy $\pi_{k+1}(s)$ by assigning probabilities to each action $\pi_{k+1}(s,a)$ proportional to its thresholded cumulative regret $\textsc{tcreg}_i(k, \pi, a)$. There is one main difference: this improvement step is not (necessarily) a contraction on any distance to an optimal policy. However, the average policy $\bar{\pi}_k$ {\it does} converges due to the Folk theorem, so in some sense the policy update operator on $\pi_k$ is improving $\bar{\pi}_k$. We give more detail on CFR in the following subsection (Appendix~\ref{sec:cfr}).

Like standard value/policy iteration, CFR requires a full state space sweep each iteration. Intead, Monte Carlo sampling can be applied to get estimated values $\tilde{v}^c$~\cite{Lanctot09mccfr}. Then the equivalent policy update operator can be applied and there are probabilistic bounds on convergence to equilibrium. 

One main crticial point is that temporal difference bootstrapping from values recursively is not possible as the Markov property does not hold in general for information states in partially-observable multiagent environments: the optimal policy $\pi_i(s,a)$ at some state $s_t$ {\it does} generally depend on the policies at other information states.

POMDPs represent hidden state using belief states. They are different from information states, as they are paired with an associated distribution over the histories.

The following table shows a mapping between most-used terms that are analogous (mostly equivalent) but used within the two separate communities:

\begin{table}[h!]
\begin{center}
\begin{tabular}{|l|l|ll|}
\hline
Computational Game Theory & Reinforcement Learning  & This paper & Prev. paper(s)  \\
\hline
Player                  & Agent & $i$ & $i$\\
Information set         & Information state & $s$ & $I$    \\
Action (or move)        & Action & $a$ & $a$ \\
History                 & State & $h$ & $h$\\
Utility                 & Reward & $u, G$ & $u$ \\
Strategy                & Policy & $\pi$ & $\sigma$ \\
Reach probability       & \footnotemark[4] & $\eta$ & $\pi$ \\
Chance event probability & Transition probability & $\cT$ & $\sigma_c$ \\
Chance                  & Nature &  & \\
Imperfect Information   & Partial Observability & & \\
Extensive-form game     & Multiagent Environment\footnotemark[5] & & \\
Simultaneous-move/Stochastic Game  &  Markov/Stochastic Game & & \\
\hline
\end{tabular}
\end{center}
\caption{A mapping of analogous terms across fields. The last two columns show nomenclature used for instances of each, compared to the previous papers from computational game theory.}
\end{table}

\footnotetext[4]{There is no precise equivalent. The closest is the on-policy distribution in episodic tasks $\mu(s)$ described in \cite[Section 9.2]{Sutton18}.}
\footnotetext[5]{Also: finite-horizon Dec-POMDP, in the cooperative setting.}

\section{Counterfactual Regret Minimization \label{sec:cfr}}

As mentioned above, Counterfactial Regret Minimization (CFR) is a policy iteration technique with a different policy improvement step. In this section we describe the algorithm using the terminology as defined in this paper.
Again, it is presented in a slightly different way than from previous papers to emphasize the elements of policy iteration.
For an overview with more background, see~\cite[Chapter 2]{lanctot13phdthesis}. For a thorough introductory tutorial, including backgound and code exercises see~\cite{Neller13cfrnotes}.

\begin{algorithm2e}[h!]
\SetKwInOut{Input}{input}\SetKwInOut{Output}{output}
\Input{~~~$K$ -- number of iterations; $\pi^0$ -- initial uniform joint policy}
Initialize table of values $v_i^c(\pi, s, a) = v_i^c(\pi, s) = 0$ for all $s,a$ \;
Initialize cumulative regret tables $\mbox{CREG}(s,a) = 0$ for all $s,a$\;
Initialize average policy tables $S(s,a) = 0$ for all $s,a$\;
 \;
\textsc{PolicyEvalTreeWalk}(joint policy $\pi$, history $h$, player reach probs $\vec{\eta}$, chance reach $\eta_c$): \;
\If{$h$ is terminal}{
    {\bf return} utilites (episode returns) $\vec{u} = (u_i(h))$ for $i \in \cN$ \;
}
\ElseIf{$h$ is a chance node}{
    {\bf return} $\sum_{a \in \cA(h)}\Pr(ha | h) \cdot \textsc{PolicyEvalTreeWalk}(ha, i, \vec{\eta}, \Pr(ha | h) \cdot \eta_c)$ \;
}\Else{
    Let $i$ be the player to play at $h$ \;
    Let $s$ be the information state containing $h$ \;
    $\vec{u} \leftarrow \vec{0}$ \;
    \For{legal actions $a \in \cA(h)$}{
        $\eta_{-i} \leftarrow \eta_c \cdot \Pi_{j \not= i} \eta_j$ \;
        $\vec{\eta}' \leftarrow \vec{\eta}$ \;
        $\eta_i' \leftarrow \eta_i' \cdot \pi_i(s, a)$ \;
        $\vec{u}_a \leftarrow \textsc{PolicyEvalTreeWalk}(ha, \vec{\eta}', \eta_c)$ \;
        $v_i^c(\pi, s, a) \leftarrow v_i^c(\pi, s, a) + u_{a,i}~~~~~(i^{th} \mbox { component})$ \;
        $S(s,a) \leftarrow S(s,a) + \eta_i \cdot \pi(s,a)~~~~~(\mbox{policy improvement on average policy } \bar{\pi}$) \;
        $\vec{u} \leftarrow \vec{u} + \pi(s,a) \vec{u}_a$ \;
    }
    {\bf return} $\vec{u}$ \;
}
 \;
\textsc{PolicyEvaluation}(iteration $k$): \;
Let $h$ be the initial empty history \;
PolicyEvalTreeWalk($\pi^k$, $h$, $\vec{1}$, $\vec{1}$) \;
\For{all $s$}{
   Let $i$ be the player to play at $s$ \;
   $v_i^c(\pi^k, s) = \sum_b \pi^k(s,b) v_i^c(\pi^k, s, b)$ \;
}
 \;
\textsc{RegretMatching}(information state $s$): \;
Define thresholded cumulative TCREG$(s,a) = (\mbox{CREG}(s,a))^+$ \;
$d \leftarrow \sum_{b} \mbox{TCREG}(s,b)$ \;
\For{$a \in \cA(s)$}{
  $\pi(s,a) \leftarrow \frac{\mbox{TCREG}(s,a)}{d}$ if $d > 0$ otherwise $\frac{1}{|\cA(s)|}$ \;
}
{\bf return} $\pi(s)$ \;
 \;
\textsc{PolicyUpdate}(iteration $k$): \;
\For{all $s$}{
    Let $i$ be the player to play at $s$ \;
    \For{$a \in \cA(s)$}{
        $\mbox{CREG}(s,a) \leftarrow \mbox{CREG}(s,a) + (v_i^c(\pi^k, s, a) - v^c_i(\pi^k, s))$ \;
    }
    $\pi^{k+1}(s) \leftarrow \mbox{RegretMatching}(s)$ \label{alg:cfr-rm}\;
}
 \;
\For{$k \in \{1, 2, \cdots, K \}$}{
   Set all the counterfactual values $v_i^c(\pi^k, s, a) = v_i^c(\pi^k, s) = 0$ \;
   PolicyEvaluation($k$) \;
   PolicyUpdate($k$) \;
}
\For{all $s$}{
  $\bar{\pi}^T(s,a) = \frac{S(s,a)}{\sum_b S(s,b)}$ \;
}
{\bf return} $\bar{\pi}^T$ \;
\caption{Vanilla CFR \label{alg:cfr}}
\end{algorithm2e}

\section{Regret-based Policy Gradients: Algorithm Pseudo-Code \label{sec:pseudocode}}

The algorithm is similar in form to A2C~\cite{Mnih2016asynchronous}. The differences are:
\begin{enumerate}
\item Gradient {\it descent} is used with $\nabla_{\btheta}^{\textsc{RPG}}$ instead of gradient ascent in $\nabla_{\btheta}^{\textsc{QPG}}$, 
$\nabla_{\btheta}^{\textsc{RMPG}}$, and A2C.
\item An (action,value) $q$-function critic is used in place of the usual state baseline $v$.
\end{enumerate}

The pseudo-code is presented in Algorithm~\ref{alg:rpg}.

\begin{algorithm2e}[h!]
\SetKwInOut{Input}{input}\SetKwInOut{Output}{output}
\Input{$\pi$ -- policy; $s_0$ -- initial state}
\Repeat{$T > T_{max}$}{
  Reset gradients: $d \btheta \leftarrow 0$, and $d \bw \leftarrow 0$. \;
  $t_{start} \leftarrow t$ \;
  \Repeat{terminal $s_t$ {\bf or} $t - t_{start} = t_{max}$}{
    Sample $a_t \sim \pi(\cdot~|~s_t, \btheta)$ \;
    Take action $a_t$ and receive reward $r_t$ and $s_{t+1}$ \;
    $t \leftarrow t + 1$ \;
    $T \leftarrow T + 1$ \;
  }
  $G \leftarrow \left\{ \begin{array}{ll}
         0 & \mbox{if $s_t$ is terminal};\\
        \sum_{a \in A} \pi(a~|~s_t, \btheta) Q(s, a; \bw) & \mbox{otherwise}.\end{array} \right.$ \;
  \For{$i \in \{t - 1, \ldots, t_{start} \}$}{
    $G \leftarrow r_i + \gamma G$ \;
    Acc. policy gradients: $d \btheta \leftarrow d \btheta + \delta$, \mbox{where $\delta$ is one of $\{\nabla_{\btheta}^{\textsc{QPG}}$,
    $\nabla_{\btheta}^{\textsc{RPG}}$, $\nabla_{\btheta}^{\textsc{RMPG}} \}$ from Sec. \ref{sec:rpg}}\;
    Acc. $q$-value function gradients: $d \bw \leftarrow d \bw + \nabla_\bw (G - q(s_i, a_i; \bw))^2$ \;
  }
  Update critic: $\bw \leftarrow \bw - \alpha d \bw$ \; \label{alg:policy-update}
  Update actor: $\btheta \leftarrow \btheta + \alpha d \btheta$ \;
}
\caption{Generalized Advantage Actor-Critic with (state,action) critics. \label{alg:rpg}}
\end{algorithm2e}

In this paper we focus on episodes of bounded length and $t_{max}$ is greater than the maximum number of moves per episode. So there is no TD-style bootstrapping from other values. In environments with longer episodes, it might be necessary to truncate the sequence lengths as is common in Deep RL.

\section{Analysis of Regret Dynamics in Matrix Games \label{sec:app-dynamics}}

In Tables \ref{tab:mp}, \ref{tab:rps}, and  \ref{tab:rpsbias} we show the three games under study, i.e., matching pennies (MP), rock-paper-scissors (RPS), and a skewed version of the latter, called bias rock-paper-scissors (bRPS) from \cite{Bosansky16Algorithms}.
In Figures \ref{MP}, \ref{RPS} and \ref{RPSbias} we illustrate several dynamics in these respective games.
More precisely, we show classical Replicator Dynamics (RD) as a reference point (row a), RPG Dynamics (row b), and time average RPG Dynamics plots (row c), sorted row by row. 
As can be observed from the figures, and as is well known, the RD cycle around the mixed Nash equilibrium (indicated by a yellow dot) in all three games, see row (a). The RPG dynamics cycle as well, though in a slightly different manner than RD as can be seen from row (b). Finally, in row (c) we show the average time RPG dynamics. Interestingly these plots show that in all three cases the learning dynamics converge to the mixed Nash equilibrium. These final plots illustrate that the average intended behavior of RPG converges to the mixed Nash equilibrium and that the RPG algorithm is regret-minimizing in these specific normal-form games \cite{Hofbauer09}.

\subsection{Normal Form Games Dynamical Systems}
\label{sec:NFG_DynamicalSystem}
Here we present the dynamical systems that describe each policy gradient update rule in two-player matrix games. For futher detail on the construction and analysis of these dynamical systems, see~\cite{BloembergenTHK15,HofbauerSigmund98}.

Let us recall the updated we consider in this paper.

QPG:
\begin{equation}
    \nabla_{\btheta}^{\textsc{QPG}}(s) = \sum_a [\nabla_\theta \pi(s, a; \btheta)] \left(q(s,a; \bw) - \sum_b \pi(s, b; \btheta)  q(s,  b, \bw)\right).
\label{eq:qac}
\end{equation}

RPG:
\begin{equation}
    \nabla_{\btheta}^{\textsc{RPG}}(s) = -\sum_a \nabla_\theta \left(q(s,a; \bw) - \sum_b \pi(s, b; \btheta) q(s,b; \bw)\right)^+.
\label{eq:rpg}
\end{equation}

RMPG:
\begin{equation}
    \nabla_{\btheta}^{\textsc{RMPG}}(s) = \sum_a [\nabla_\theta \pi(s, a; \btheta)] \left(q(s,a; \bw) - \sum_b \pi(s, b; \btheta) q(s, b, \bw)\right)^+.
\label{eq:rmpg}
\end{equation}

Let us consider that the game is in normal form and let us suppose that the policy is only parametrized by logits. The parameter will be $\btheta = (\theta_a)_a$ and $\pi(a; \btheta) = \frac{\exp(\theta_a)}{\sum_a \exp(\theta_a)}$ in a state less game. It follows that:

\begin{equation}
    \frac{d \pi(a;\btheta)}{d \theta_b} = 1_{a=b} \pi(a;\btheta) - \pi(a;\btheta)\pi(b;\btheta)
\label{eq:diff_pi}
\end{equation}

\begin{equation}
    \dot{\pi}(a; \btheta) = \sum \limits_{b} \frac{d \pi(a;\btheta)}{d \theta_b} \dot{\theta}_b= \pi(a; \btheta)(\dot{\theta}_a - \sum \limits_b \pi(b; \btheta) \dot{\theta}_b)
\label{eq:diff_t}
\end{equation}

\subsubsection{QPG}

The dynamical system followed by QPG on a normal form game can be written as follow:
\begin{equation}
    \dot{\btheta} = \sum_a [\nabla_\theta \pi(a; \btheta)] \left(q(a; \bw) - \sum_b \pi(b; \btheta)  q(b, \bw)\right)
\end{equation}
\begin{equation}
    \dot{\btheta} = \sum_a \nabla_\theta \pi(a; \btheta) q(a; \bw) \textit{ because $\sum_a \pi(a; \btheta) = 1$}
\label{eq:qac_no_baseline}
\end{equation}
\begin{equation}
    \dot{\theta_b} = \sum_a \frac{d \pi(a; \btheta)}{d \theta_b}  q(a; \bw) = \pi(b; \btheta) \left(q(b; \bw) - \sum \limits_a \pi(a; \btheta) q(a; \bw) \right) = \pi(b; \btheta) A(b, \btheta, \bw)
\end{equation}

Final dynamical system:
\begin{align}
    &\dot{\pi}(a; \btheta) = \pi(a; \btheta) \left(\pi(a; \btheta) A(a, \btheta, \bw) - \sum \limits_b \pi(b; \btheta)^2 A(b, \btheta, \bw) \right)\\
    & \quad \textit{ where $A(a, \btheta, \bw) = q(a; \bw) - \sum \limits_b \pi(b; \btheta) q(b; \bw)$}
\end{align}

\subsubsection{RPG}

\begin{equation}
    \nabla_{\btheta}^{\textsc{RPG}}(s) = -\sum_a \nabla_\theta \left(q(s,a; \bw) - \sum_b \pi(s, b; \btheta) q(s,b; \bw)\right)^+
\end{equation}
\begin{equation}
    \nabla_{\btheta}^{\textsc{RPG}}(s) = -\sum_a \nabla_\theta 1_{A(a, \btheta, \bw) \geq 0}\left(q(s,a; \bw) - \sum_b \pi(s, b; \btheta) q(s,b; \bw)\right)
\end{equation}
\begin{equation}
    \nabla_{\btheta}^{\textsc{RPG}}(s) = -\sum_a 1_{A(a, \btheta, \bw) \geq 0} \nabla_\theta \left(q(s,a; \bw) - \sum_b \pi(s, b; \btheta) q(s,b; \bw)\right)
\end{equation}
\begin{equation}
    \nabla_{\btheta}^{\textsc{RPG}}(s) = \sum_a 1_{A(a, \btheta, \bw) \geq 0} \sum_b \nabla_\theta \pi(b;\btheta) q(b;\bw)
\end{equation}
\begin{equation}
    \nabla_{\btheta}^{\textsc{RPG}}(s) = \underbrace{\left(\sum_a 1_{A(a, \btheta, \bw) \geq 0}\right)}_{\textrm{$n_{a+}$}} \underbrace{\sum_b \nabla_\theta \pi(b;\btheta) q(b;\bw)}_{\textrm{$\nabla_{\btheta}^{\textsc{QPG}}$ from~\ref{eq:qac_no_baseline}}}
\label{eq:rpg_logit}
\end{equation}
\begin{equation}
    \nabla_{\btheta}^{\textsc{RPG}}(s) = n_{a+} \nabla_{\btheta}^{\textsc{QPG}}
\label{eq:rpg_logit}
\end{equation}
it falls that the dynamical system of RPG is:
\begin{equation}
    \dot{\pi}(a; \btheta) = n_{a+} \left[ \pi(a; \btheta) \left(\pi(a; \btheta) A(a, \btheta, \bw) - \sum \limits_b \pi(b; \btheta)^2 A(b, \btheta, \bw) \right) \right]
\end{equation}

\begin{table}[h!]
\begin{minipage}{0.30\textwidth}
   \begin{game}{2}{2}[][]
   	  &  $H$     &  $T$   \\
   	 H  &    $1,-1$      & $-1,1$ \\
   	 T &  $-1,1$ & $1,-1$ \\
   	 \\
   \end{game}
   \vspace{0.2cm}
   \caption{\footnotesize Matching Pennies.}
   \label{tab:mp}
\end{minipage}
\qquad
\begin{minipage}{0.30\textwidth}
   \begin{game}{3}{3}[][]
   	  &  $R$     &  $P$   & $S$  \\
   	 $R$  &    $0$      & $-1$ & $1$\\
   	 $P$ &  $1$ & $0$ & $-1$\\
   	 $S$ & $-1$ & $1$ & $0$\\
   \end{game}
   \vspace{0.2cm}
   \caption{\footnotesize Rock-Paper-Scissors.}
   \label{tab:rps}
\end{minipage}
\qquad
\begin{minipage}{0.30\textwidth}
\begin{game}{3}{3}[][]
   	  &  $R$     &  $P$   & $S$  \\
   	 $R$  &    $0$      & $-0.25$ & $0.5$\\
   	 $P$ &  $0.25$ & $0$ & $-0.05$\\
   	 $S$ & $-0.5$ & $0.05$ & $0$\\
   \end{game}
   \vspace{0.2cm}
  \caption{\footnotesize Bias RPS.}\label{tab:rpsbias}
   \end{minipage}
\end{table}

\begin{figure*}[h!]
\vspace{-0.3cm}
\centering
\begin{minipage}{0.33\textwidth}
\caption{\footnotesize Matching Pennies}\label{MP}
    \begin{subfigure}{\textwidth}
    \centering
    \includegraphics[width=\textwidth]{MatchingPenniesRDdynamics.png}
    \caption{\footnotesize Replicator Dynamics}
    \end{subfigure}\\
    \begin{subfigure}{\textwidth}
    \centering
    \includegraphics[width=\textwidth]{MatchingPenniesRPGDynamics.png}
    \caption{\footnotesize RPG dynamics}
    \end{subfigure} \\ 
    \begin{subfigure}{\textwidth}
    \centering
    \includegraphics[width=\textwidth]{MatchingPenniesAverageRPGdynamics.png}
   \caption{\footnotesize Average RPG dynamics}
    \end{subfigure}
    
\end{minipage}%
\hfill
\begin{minipage}{.33\textwidth}
\caption{\footnotesize Rock-Paper-Scissors}\label{RPS}
    \begin{subfigure}{\textwidth}
    \centering
    \includegraphics[width=\textwidth]{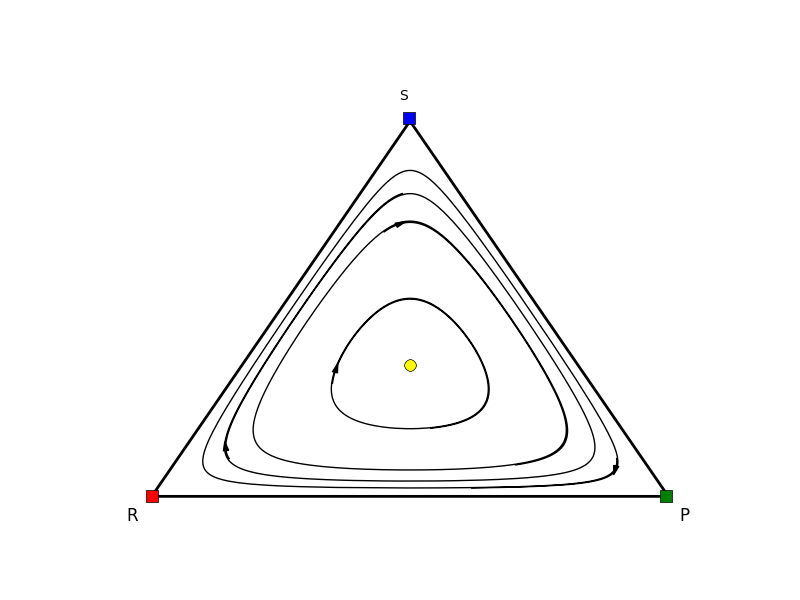}
    \caption{\footnotesize Replicator Dynamics}
    \end{subfigure}\\
    \begin{subfigure}{\textwidth}
    \centering
    \includegraphics[width=\textwidth]{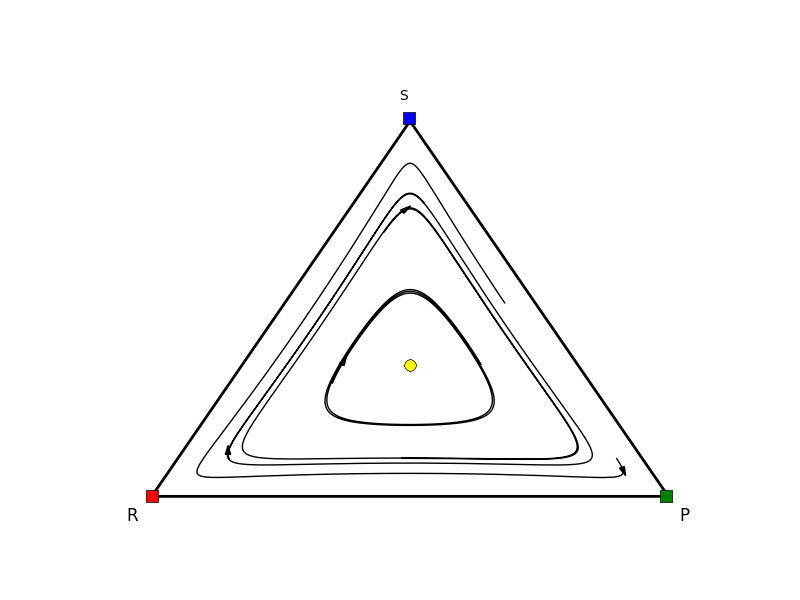}
    \caption{\footnotesize RPG dynamics}
    \end{subfigure}\\ 
    \begin{subfigure}{\textwidth}
    \centering
    \includegraphics[width=\textwidth]{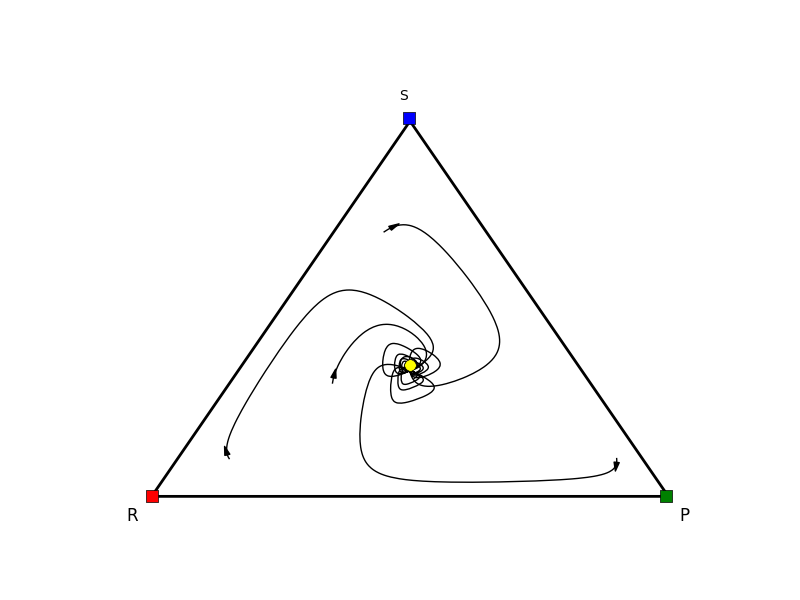}
    \caption{\footnotesize Average RPG dynamics}
    \end{subfigure}
    
\end{minipage}
\hfill
\begin{minipage}{.33\textwidth}
\caption{\footnotesize Bias Rock-Paper-Scissors}\label{RPSbias}
    \begin{subfigure}{\textwidth}
    \centering
    \includegraphics[width=\textwidth]{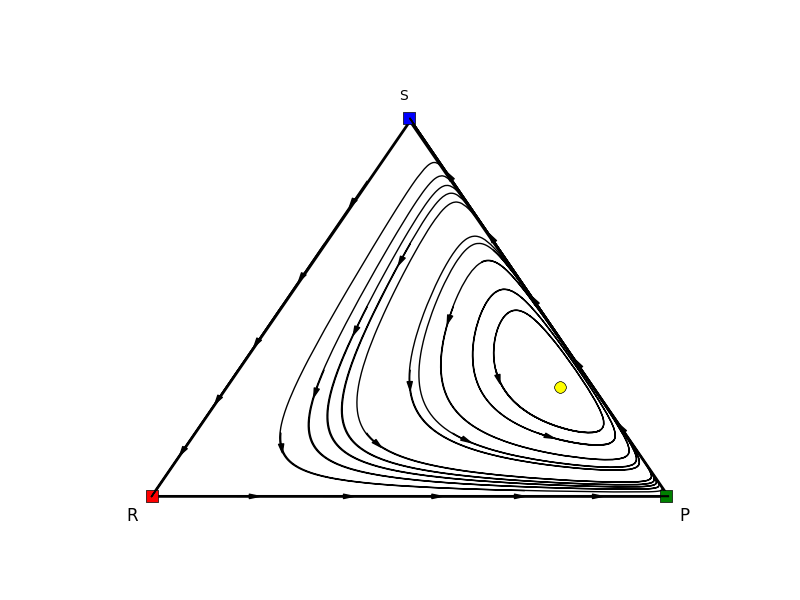}
    \caption{\footnotesize Replicator Dynamics}
    \end{subfigure}\\
    \begin{subfigure}{\textwidth}
    \centering
    \includegraphics[width=\textwidth]{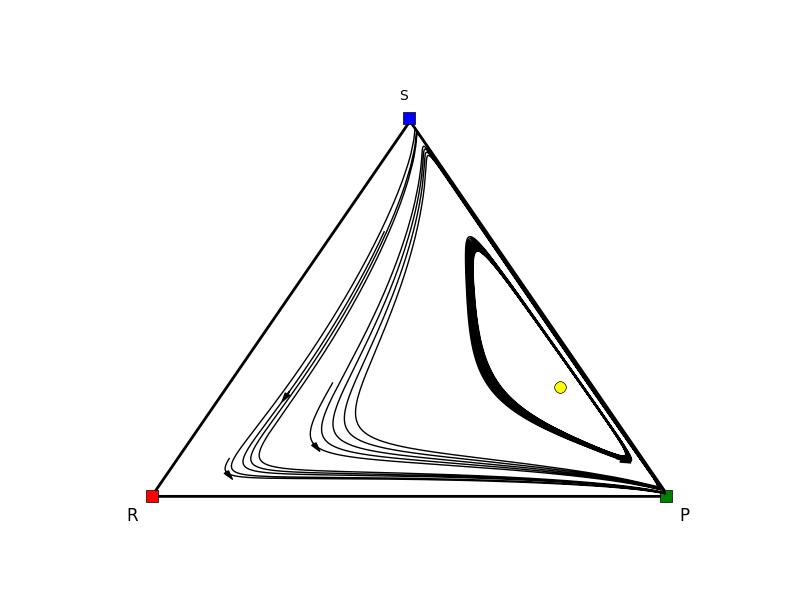}
    \caption{\footnotesize RPG dynamics}
    \end{subfigure}\\ 
    \begin{subfigure}{\textwidth}
    \centering
    \includegraphics[width=\textwidth]{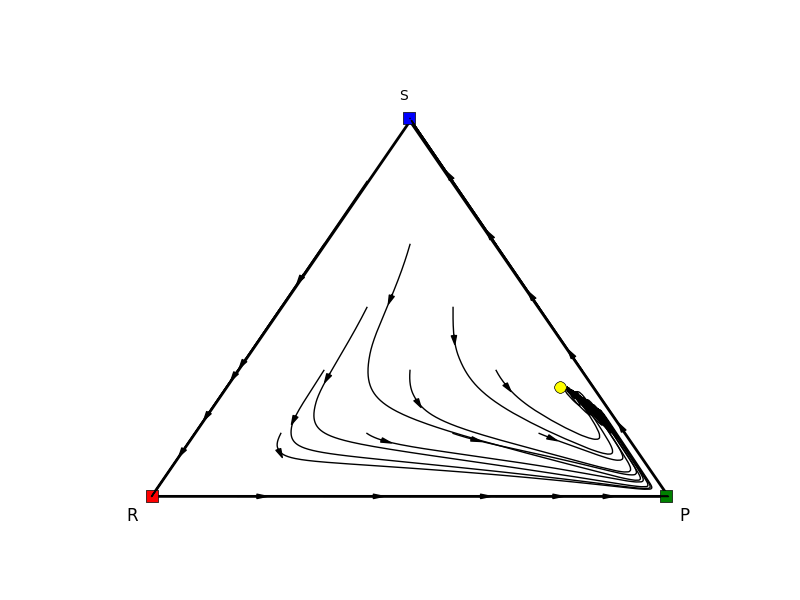}
    \caption{\footnotesize Average RPG dynamics}
    \end{subfigure}
    
\end{minipage}
\end{figure*}

\subsection{Generalised Rock-Paper-Scissors Game}

For the sake of completeness we also looked at the behavior of the dynamics in the generalised Rock-Paper-Scissors (gRPS) game \cite{Qian,Shapley}. More precisely, the gRPS game can be described as illustrated in Table \ref{tab:gRPS}.

\begin{table}[h!tb]
	\centering
   \begin{game}{3}{3}[][]
   	  &  $R$     &  $P$   & $S$  \\
   	 $R$  &    $1$      & $0$ & $2$\\
   	 $P$ &  $2$ & $1$ & $0$\\
   	 $S$ & $0$ & $2$ & $1$\\
   \end{game}
   \vspace{0.2cm}
   \caption{\footnotesize Generalized Rock-Paper-Scissors.}
   \label{tab:gRPS}
\end{table}

We describe the dynamics in this game for replicator dynamics, RPG dynamics, and both replicator and RPG dynamics as average time dynamics plots. As in the RPS game, the replicator dynamics and RPG dynamics cycle around the Nash equilibrium, and the average time replicator dynamics and average time RPG dynamics converge to the Nash equilibrium, as illustrated in Figures \ref{fig:genRD}, \ref{fig:genRPG}, \ref{fig:genavRD} and \ref{fig:genavRPG}. A more detailed description on the convergence properties of replicator equations in this game can be found in \cite{Qian}. 

\begin{figure*}[!tbp]
\vspace{-0.3cm}
\centering
\begin{minipage}{0.32\textwidth}
    \centering
    \includegraphics[width=\textwidth]{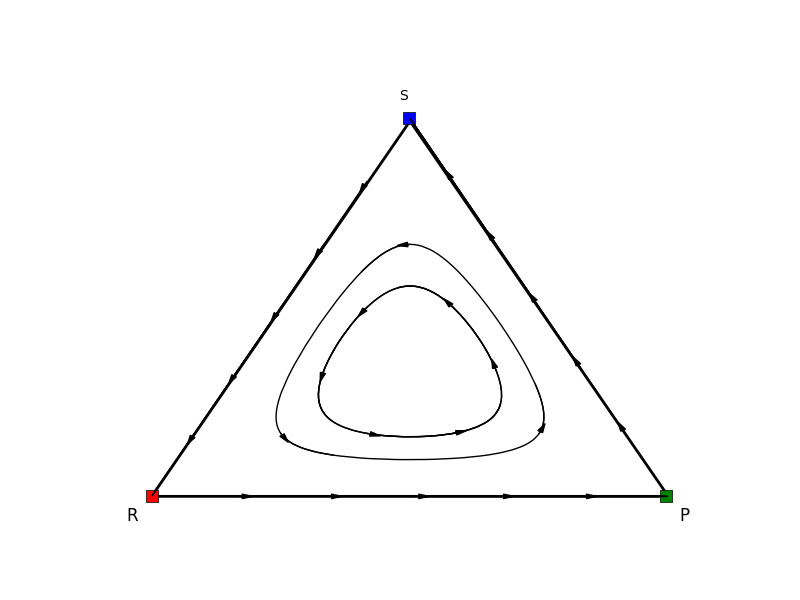}
    \caption{\footnotesize Replicator Dynamics in the gRPS Game}\label{fig:genRD}
    \end{minipage}
    ~~~
    \begin{minipage}{0.32\textwidth}
    \centering
    \includegraphics[width=\textwidth]{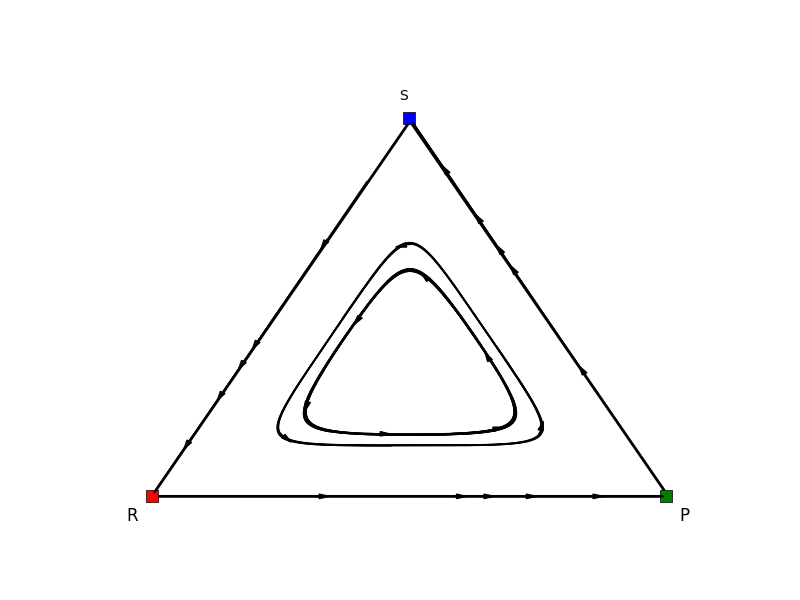}
    \caption{\footnotesize RPG dynamics in the gRPS Game}\label{fig:genRPG}
    \end{minipage} 
\end{figure*}
\begin{figure*}[!tbp]
\vspace{-0.3cm}
\centering
    \begin{minipage}{0.32\textwidth}
    \centering
    \includegraphics[width=\textwidth]{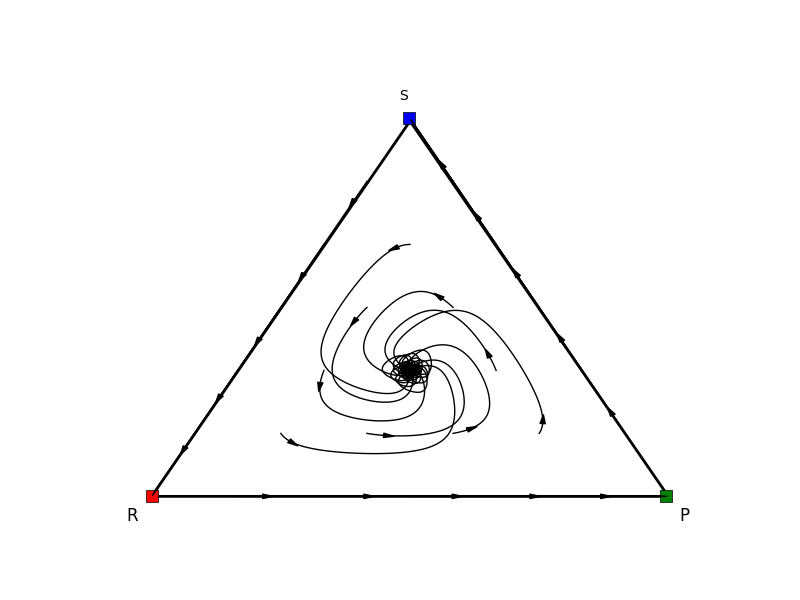}
   \caption{\footnotesize Avg. Replicator Dynamics in the gRPS Game}\label{fig:genavRD}
   \end{minipage}
   ~~~
   \begin{minipage}{0.32\textwidth}
    \centering
    \includegraphics[width=\textwidth]{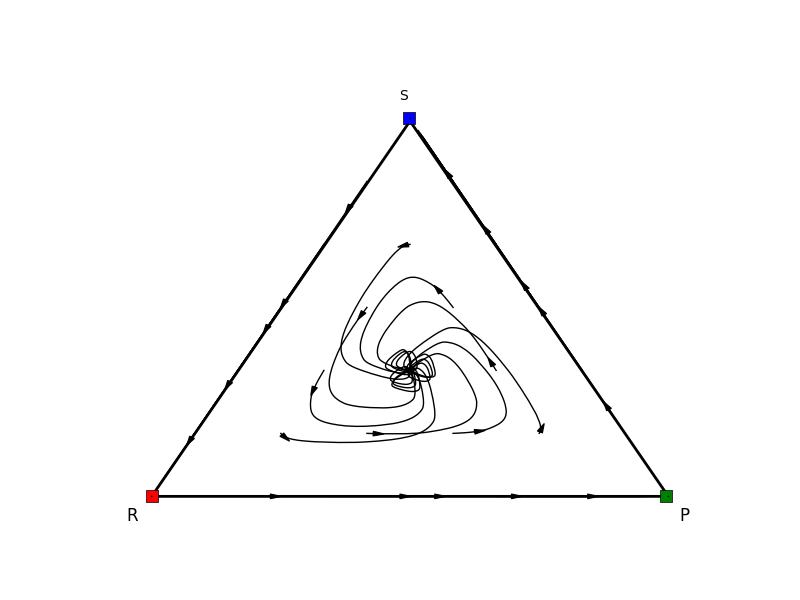}
   \caption{\footnotesize Avg. RPG Dynamics in the gRPS Game}\label{fig:genavRPG}
    \end{minipage}
  \end{figure*}

\newpage

\section{Sequential Partially-Observable Case \label{sec:app-sequential}}

Let $P$ be defined as in Theorem~\ref{thm:pg-cfr-conv}. We first define the four update rules that we will discuss in this section.
On iteration $k$, at state $s$, the update to the policy parameters are:
\[
\mbox{Projected PGPI}: \theta^{k+1}_{s, \cdot} \leftarrow P( \{ \theta^k_{s,a} + \alpha_{s,k} \frac{\partial}{\partial \theta^k_{s,a}} J^{PG}(\pi_{\btheta^k}) \}_a )
\]
\[
\mbox{Projected ACPI}: \theta^{k+1}_{s, \cdot} \leftarrow P( \{ \theta^k_{s,a} + \alpha_{s,k} \frac{\partial}{\partial \theta^k_{s,a}} J^{AC}(\pi_{\btheta^k}) \}_a )
\]
\[
\mbox{Projected Strong PGPI}: \theta^{k+1}_{s, \cdot} \leftarrow P( \{ \theta^k_{s,a} + \alpha_k \frac{\partial}{\partial \theta^k_{s,a}} J^{PG}(\pi_{\btheta^k}, s) \}_a )
\]
\[
\mbox{Projected Strong ACPI}: \theta^{k+1}_{s, \cdot} \leftarrow P( \{ \theta^k_{s,a} + \alpha_k \frac{\partial}{\partial \theta^k_{s,a}} J^{AC}(\pi_{\btheta^k}, s) \}_a )
\]

Tabular policies are represents in behavioral strategy form: a probability is a weight $\theta_{s,a}$ per state-action, where the weights obey simplex constraints: $\forall s, \sum_a \theta_{s,a} = 1$.

In turn-based games, the gradient of a tabular policy, $\nabla_{\btheta} \pi_\theta$ is then simply a sum of partial derivates with respect to each specific weight $\theta_{s,a}$.

The score function $J^{PG}(\pi_{\btheta}) = \bE_{\rho \sim \bpi}[ G_0~|~S_0 = s_0 ] = \sum_{z \in \cZ} \eta^\bpi(z) G_{i,z}$. The contribution of some $\theta_{s,a}$ to the gradient is:
\begin{eqnarray*}
\frac{\partial J^{PG}(\pi_\theta)}{\partial \theta_{s,a}}& = &
\frac{\partial }{\partial \theta_{s,a}} \sum_{z \in \cZ} \eta^\bpi(z) G_{i,z} \\
 & = & \frac{\partial }{\partial \theta_{s,a}} \left( \sum_{h,z \in \cZ(s,a)} \eta^\bpi(z) G_{i,z} \right)~~~~~~\text{since other terminal histories do not contain $\theta_{s,a}$}\\
  & = & \sum_{h,z \in \cZ(s,a)}{\frac{\partial }{\partial \theta_{s,a}} \eta^\bpi(z) G_{i,z}} \\
  & = & \sum_{h,z \in \cZ(s,a)}{\eta^\bpi(h) \eta^\bpi(ha, z) G_{i,z}}\\
  & := & v_{\eta,\pi}(s,a)~~~~\text{(definition)}
\end{eqnarray*}

Here, we define $v_{\eta,\pi}(s,a)$ as the (reach-weighted) portion of the overall expected value contributed by action $a$ at information state $s$.
The weight $\sum_{h,z \in \cZ(s,a)}{\eta^\bpi(h)}$ is analogous to the on-policy distribution $\mu$ in the standard policy gradient theorem.
Each component of the gradient treats the other $\theta_{s',a'}$ as constant, increasing the local expected value contributed at $s$, which is just $\pi(s,\ba) \cdot v_{\eta,\pi}(s, \ba)$, can be optimized independently for the purposes of taking a single gradient step. The result is that the problem can be decomposed into a per-state optimization problem, for the purposes of a single policy gradient update. This is a direct consequence of the tabular representation and perfect recall.

We then observe that ACPI can implement some form of Generalized Infinitesimal Gradient Ascent (GIGA) algorithm of Zinkevich~\cite{Zinkevich03Online}, an application of greedy projection or now also called online gradient descent~\cite{Hazan15OCO}.
The idea here is that there is an online convex program with a convex cost function $c_k$ at each step $k$. The optimization proceeds by moving the point $x$ (\ie the policy) following the gradient of $c_k$ at at step $k$ and projecting back into the feasible set (of simplices) greedily after each gradient step.
GIGA is an application of online gradient descent to repeated games;
Zinkevich shows that GIGA minimizes regret against this adversary, defining a new OCP after each play.

In our case, we have {\it local online convex programs} $\textsc{OCP}(s,k)$, one at each $s$, and a separate instantiation of $\textsc{GIGA}(s)$ at $s$ that solves this local OCP at each $s$.
Each problem is locally convex, and the adversary is the policy $\pi_{-i}$ for states outside of $s$. We use this as a basis to prove Theorem~\ref{thm:pg-cfr-conv}. 

Our construction essentially shows that ACPI is similar to CFR  except with the policy update rule (RegretMatching on Algorithm~\ref{alg:cfr}, line~\ref{alg:cfr-rm}) replaced by GIGA($s$).
We then show that PGPI can be treated as a special case of this overall argument.

\begin{definition}
\label{def:giga-s}
State-local Generalized Infinitesimal Gradient Ascent ($\textsc{GIGA}(s)$) (an adaptation of \cite[Algorithm 2]{Zinkevich03Online} proceeds as follows.
Let $\btheta_s = ( \theta_{s,a'}$ for $a' \in \cA$ ) be the policy parameters at $s$.
Initialize $\btheta_s$ arbitrarily according to simplex constraints.
Choose a sequence of learning rates $\{ \alpha_1, \alpha_2, \cdots, \alpha_K \} $, and repeat for $k \in \{ 1, 2, \cdots, K \}$:
\begin{enumerate}
\item Choose $a$ according to $\pi^k(s)$
\item Observe $\pi_{-i}$, and update the local policy $\pi^k(s)$:
    \begin{enumerate}
    \item $y^{k+1} = \btheta_{s} + \alpha_{s,k} \bv_{\eta,\pi}(s,\cdot)$
    \item $\btheta_s = P(y^{t+1})$,
    \end{enumerate}
\end{enumerate}
where the projection $P$ is defined as in Theorem \ref{thm:pg-cfr-conv}.
\end{definition}

Note that $v_{\eta,\pi}(s)$ here is the gradient wrt $\theta_{s,a}$ of the score function from above.
To prove Theorem \ref{thm:pg-cfr-conv}, we will make use of a few lemmas.

\begin{lemma}
\label{lem:score-comp}
The value of the local component $\frac{\partial J^{PG}(\pi_\theta)}{\partial \theta_{s,a}} = v_{\eta,\pi}(s,a) = \eta^\pi_i(s) v_i^c(\pi, s, a)$.
\end{lemma}
\begin{proof}
\begin{eqnarray*}
v_{\eta,\pi}(s,a) & = & \sum_{h,z \in \cZ(s,a)}{\eta^\bpi(h) \eta^\bpi(ha, z)} G_{i,z}~~~~~~~~~~~~~~\text{as defined above}\\
  & = & \sum_{h,z \in \cZ(s,a)}{\eta_i^{\bpi}(h) \eta_{-i}^{\bpi}(h) \eta^\bpi(ha, z)} G_{i,z} \\
  & = & \eta_i^{\bpi}(s) \sum_{h,z \in \cZ(s,a)}{\eta_{-i}^{\bpi}(h) \eta^\bpi(ha,z)} G_{i,z}~~~~\text{from perfect recall} \\
  & = & \eta_i^{\bpi}(s) v_i^c(\pi, s, a)
\end{eqnarray*}
\end{proof}

The following lemma shows how the advantage is related to the local value GIGA$(s)$ has in its update rule:

\begin{lemma}
\label{lem:adv}
The advantage
\[
q_\pi(s, a) - \sum_b{\pi(s,b) q_\pi(s,b)} = \frac{v_{\eta,\pi}(s,a)}{\eta^{\pi}_i(s)\cB_{-i}(\pi,s)} - \frac{v_i^c(\pi,s)}{\cB_{-i}(\pi,s)}.
\]
\end{lemma}
\begin{proof}
\begin{eqnarray*}
q_\pi(s, a) - \sum_b{\pi(s,a) q_\pi(s,b)}  & = & \frac{ v_i^c(\pi, s, a) - v^c_i(\pi, s)}{\cB_{-i}(\pi, s)}~~~~~~~~~~\text{from Section \ref{sec:rpg}}\\
& = & \frac{ v_i^c(\pi, s, a) }{\cB_{-i}(\pi, s)} - \frac{v^c_i(\pi, s)}{\cB_{-i}(\pi, s)}\\
& = & \frac{v_{\eta,\pi}(s,a)}{\eta^{\pi}_i(s) \cB_{-i}(\pi,s)} - \frac{v_i^c(\pi,s)}{\cB_{-i}(\pi,s)}~~~~~~~~~~~~~~~~~~\text{by Lemma~\ref{lem:score-comp}}
\end{eqnarray*}
\end{proof}
We require one more property about projections onto simplices:
\begin{lemma}
\label{lem:projection}
Define the simplex $\Delta = \{ \bx \in \Re^N : \sum_i x_i = 1\}$, and for some $\by \in \Re^N$ the $\ell_2$ projection $P(\by) = \argmin_{\bx \in \Delta} \Vert \bx - \by \Vert_2$. If $k$ is any real constant, then 
\[
P(\by - k\bone) = P(\by).
\]
\end{lemma}
\begin{proof}
\begin{eqnarray*}
P(\by - k\bone) & = & \argmin_{\bx \in \Delta} \Vert \bx - (\by - k \bone) \Vert_2 \\
 & = & \argmin_{\bx \in \Delta} \Vert \bx + k \bone - \by \Vert_2 \\ 
 & = & \argmin_{\bx \in \Delta} \sqrt{\sum_i^N (x_i + k - y_i)^2} \\
  & = & \argmin_{\bx \in \Delta} \sqrt{\sum_i^N (x_i^2 - 2 x_i y_i + y_i^2 + k^2 + 2k x_i - 2k y_i)} \\
  & = & \argmin_{\bx \in \Delta} \sqrt{\sum_i^N ((x_i - y_i)^2 + k^2 + 2k x_i - 2k y_i)} \\
  & = & \argmin_{\bx \in \Delta} \sqrt{\sum_i^N (x_i - y_i)^2 + \sum_i^N k^2 + \sum_i^N 2k x_i - \sum_i^N 2k y_i} \\
  & = & \argmin_{\bx \in \Delta} \sqrt{\sum_i^N (x_i - y_i)^2 + Nk^2 + 2k - 2k \sum_i^N y_i}~~~~~~~~\text{since $\bx \in \Delta$} \\
  & = & \argmin_{\bx \in \Delta} \sqrt{\sum_i^N (x_i - y_i)^2}. 
\end{eqnarray*}
The last line follows because $\sum_i^N y_i$ is constant when minimizing over $\bx$, and the functions $\sqrt{f(x)}$ and $\sqrt{f(x) + c}$, for some constant $c$, are minimized at the same point.
\end{proof}

We can now relate ACPI and GIGA$(s)$ using the lemmas above.
\begin{lemma}
\label{lem:acpi-giga}
Running projected ACPI (or projected PGPI) with learning rate $\alpha_{s,k} = k^{-\frac{1}{2}}$
is equivalent to running GIGA$(s)$ at each state $s$ with its required learning rate of $k^{-\frac{1}{2}}$; as a result, the total (local) regret defined using the values $v_{\eta,\pi}$ at each state $s$ after $K$ steps is at most $\sqrt{K} + \left( \sqrt{K} - \frac{1}{2} \right) | \cA | (\Delta r)^2$, where $\Delta r = \max r_i - \min r_i$.
\end{lemma}
\begin{proof}
We first prove the statement for ACPI.
If we rewrite the ACPI update equations statement of the theorem using the gradient as derived at the start of this subsection, after cancelling terms we get:
\[
\theta_{s,\cdot} \leftarrow P( \{ \theta_{s,a} + k^{-\frac{1}{2}} v_{\eta,\pi}(s,a) - k^{-\frac{1}{2}} v_{\eta,\pi}(s) \}_a )
\]
Notice that the right-most term is a constant value added to all the components of the vector $\btheta_s$.
These constants shift each component of the vector $\btheta_s$ by the same amount, which by Lemma~\ref{lem:projection} does not affect the resulting local policy $\pi^{k+1}(s)$ when projected back onto the simplex, leaving an equivalent update to the one in Definition~\ref{def:giga-s}. The total regret is then obtained by~\cite[Theorem 4]{Zinkevich03OnlineTR}.

To prove the statement is true for PGPI, we simply need to retrace our steps. 
When we rewrite the PGPI update equations, we arrive at a similar update, except missing the constant shift:
\[
\theta_{s,\cdot} \leftarrow P( \{ \theta_{s,a} + k^{-\frac{1}{2}} v_{\eta,\pi}(s,a) \}_a ).
\]
and the same logic holds as before without the constant shift over all actions $a$ at $s$.
\end{proof}

We are now ready to prove Theorem~\ref{thm:pg-cfr-conv}.

\subsection{Proof of Theorem \ref{thm:pg-cfr-conv}}
\label{sec:thm1-proof}

As a result of Lemma~\ref{lem:acpi-giga}, the regret is minimized locally. Formally, for $K$ steps the total local regret for playing the sequence of policies $\pi^0(s), \pi^1(s), \cdots, \pi^K(s)$ at $s$ is:
\[
R_i^k(s) =  \max_{a \in \cA} \sum_{k=1}^K v_{\eta,\pi}(s,a) - \bpi^k(s) \cdot \bv_{\eta,\pi^k}(s, \ba) \le \sqrt{K} + ( \sqrt{K} - \frac{1}{2})|\cA| (\Delta r)^2
\]
That is, the total regret over $k$ steps is sublinear in $K$, so the average regret locally at each $s$ approaches 0 as $k \rightarrow \infty$.

Using Lemma~\ref{lem:score-comp}, and noticing that the second term is a dot product over a vector whose components are $v_{\eta,\pi^k}(s,a)$ for each $a \in \cA$ at $s$, we can rewrite the above in terms of counterfactual values:
\[
\max_{a \in \cA} \sum_{k=1}^K \eta_i^{\pi^k}(s) v_i^c(\pi^k, s, a) - \eta_i^{\pi^k}(s) v_i^c(\pi^k,s) \le \sqrt{K} + (\sqrt{K} - \frac{1}{2})|\cA| (\Delta r)^2,
\]

Let $\eta_i^{\min} = \min_k \eta_i^{\pi^k}(s)$. If we divide both sides by this value to get an expression in terms of counterfactual values only:
\[
\max_{a \in \cA} \sum_{k=1}^K v_i^c(\pi^k, s, a) - v_i^c(\pi^k,s) \le \frac{1}{\eta_i^{\min}} \left( \sqrt{K} + (\sqrt{K} - \frac{1}{2})|\cA| (\Delta r)^2 + O(K) \right).
\]
However, this adds a term $O(K)$ since the counterfactual regrets can be negative, which means the average regret may not be guaranteed to reduce to zero asymptotically as $K \rightarrow \infty$.
This last line highlights the main difference between counterfactual regret minimization and ACPI / PGPI: the regret being minimized via GIGA$(s)$ is with
respect to a reach-weighted value, which prevents the application of the CFR theorem.
We discuss this further in Appendix~\ref{app:errata}.

\subsection{Proof of Theorem \ref{thm:strong-acpi}}

The original optimization problem concentrates on ascending the score function $J(\bpi_{\theta}) = v_{\bpi}(s_0)$. The change in policy at every state is focused on increasing only the value of the initial state, which leads to the changes at each state to be weighted by their reach probability. Hence this update is a kind of incremental policy improvement with small/careful improvement steps, which is closer in spirit to Conservative Policy Iteration~\cite{Kakade02CPI} and Maximum aposteriori Policy Optimisation~\cite{MPO}. This is in contrast to other tabular methods that perform a greedy policy improvement steps at every state.

Instead, strong ACPI changes the optimization problem to modify the policy in the direction of ascent at all the state-local values simultaneously, which is closer in form to assigning (or moving toward)  $\pi(s) = \argmax_a q(s,a)$ at every state $s$ as is done in value-based methods. 
This changes the policy updates, so we now re-derive the gradient components from the start of this section, but using state-local gradients at each state $s$. 

Here, $J^{PG}(\pi_{\btheta}, s) = \bE_{\rho \sim \bpi}[ G_t~|~S_t = s ]$. The contribution of some $\theta_{s,a}$ to the gradient is:
\begin{eqnarray*}
\frac{\partial J^{PG}(\pi_\theta, s)}{\partial \theta_{s,a}}& = &\frac{\partial }{\partial \theta_{s,a}} \left( \sum_{h,z \in \cZ(s,a)} \Pr(h~|~s) \eta^{\pi}(h,z) u_i(z) \right)\\
  & = & \frac{\partial }{\partial \theta_{s,a}} \sum_{h,z \in \cZ(s,a)} \frac{\eta_{-i}^{\pi}(h)}{\cB_{-i}(\pi, s)} \eta^{\pi}(h,z) u_i(z)~~~~~~~~~~~~~~\text{from Section~\ref{sec:rpg-seq}}\\
  & = & \sum_{h,z \in \cZ(s,a)} \frac{\eta_{-i}^{\pi}(h)}{\cB_{-i}(\pi, s)} \eta^{\pi}(ha,z) u_i(z).\\
\end{eqnarray*}
which differs from $v_{\eta,\pi}(s,a)$ by replacing the weight $\eta^\pi(h)$ by $\frac{\eta_{-i}^{\pi}(h)}{\cB_{-i}(\pi, s)}$ inside the sum.

It is now easy to verify that this value is more desirable than $v_{\eta,\pi}(s,a)$ by looking again at Lemma \ref{lem:score-comp} and using this new definition of value. Continuing from above, 
\begin{eqnarray*}
\frac{\partial J^{PG}(\pi_\theta, s)}{\partial \theta_{s,a}}& = & \ldots\\
  & = & \sum_{h,z \in \cZ(s,a)}  \frac{\eta_{-i}^{\pi}(h)}{\cB_{-i}(\pi, s)} \eta^{\pi}(ha,z) u_i(z)\\
  & = & \frac{1}{\cB_{-i}(\pi, s)} \sum_{h,z \in \cZ(s,a)} \eta_{-i}^{\pi}(h) \eta^{\pi}(ha,z) u_i(z) \\
  & = & \frac{v_i^c(\pi, s, a)}{\cB_{-i}(\pi, s)}~~~~~~~~~~~~~~\text{(which is an analog to Lemma~\ref{lem:score-comp})}.
\end{eqnarray*}
The advantage works out to be: $q_\pi(s, a) - \sum_b{\pi(s,b) q_\pi(s,b)}$
\begin{eqnarray*}
& = & \frac{ v_i^c(\pi, s, a) - v^c_i(\pi, s)}{\cB_{-i}(\pi, s)}~~~~~~~~~~\text{from Section \ref{sec:rpg}}\\
& = & \frac{ v_i^c(\pi, s, a) }{\cB_{-i}(\pi, s)} - \frac{v^c_i(\pi, s)}{\cB_{-i}(\pi, s)}\\
& = & \frac{ \partial }{\partial \theta_{s,a}} J^{PG}(\pi_{\theta},s) - \sum_b \pi(s,b) \frac{ \partial }{\partial \theta_{s,b}} J^{PG}(\pi_{\theta},s)~~~~~~~~~~\text{by the derivation above.}
\end{eqnarray*}
Now, when GIGA$(s)$ uses this new value, we can state a lemma that analogous to Lemma \ref{lem:acpi-giga}:
\begin{lemma}
\label{lem:sacpi-giga}
Running SACPI (or SPGPI) with learning rate $\alpha_{k} = k^{-\frac{1}{2}}$ is equivalent to running GIGA$(s)$ at each state $s$ with its required learning rate of $k^{-\frac{1}{2}}$; as a result, the total (local) regret defined using the values $v_{\eta, \pi}$ at each state $s$ after $K$ steps is at most $\sqrt{K} + \left( \sqrt{K} - \frac{1}{2} \right) | \cA | (\Delta r)^2$, where $\Delta r = \max r_i - \min r_i$.
\end{lemma}
The proof follows the same logic as in the proof of Lemma \ref{lem:acpi-giga}.

The proof of Theorem \ref{thm:strong-acpi} then follows very closely the steps of proof of Theorem \ref{thm:pg-cfr-conv}, but instead of the counterfactual values being weighted by $\eta_i^{\pi^k}(s)$, they are instead weighted by $\frac{1}{\cB_{-i}(\pi^k, s)}$, yielding the expression:
\[
\max_{a \in \cA} \sum_{k=1}^K \frac{1}{\cB_{-i}(\pi^k, s)} \left( v_i^c(\pi^k, s, a) - v_i^c(\pi^k,s) \right) \le \sqrt{K} + (\sqrt{K} - \frac{1}{2})|\cA| (\Delta r)^2.
\]
This has a similar problem to the proof of Theorem~\ref{thm:pg-cfr-conv}. The $\frac{1}{\cB_{-i}(\pi^k, s)}$ term cannot be removed here to get
an expression as a sum over only counterfactual regrets, despite the fact that $\frac{1}{\cB_{-i}(\pi^k, s)} \ge 1$. These coefficients scale the magnitudes
of the regrets, so in the worst-case the coefficients for the negative regrets in the sequence could be high ($\gg 1$), and low $(\approx 1)$ for the positive ones.

\section{On the similarity of QPG and RPG \label{sec:app-qpg-rpg}}

As discussed in Section~\ref{sec:rpg-seq} and in \cite[Chapter 13]{Sutton18}, subtracting the baseline does not affect the gradient. Therefore, the QPG gradient at state $s$ can be written as:
\[
\nabla_{\btheta}^{\textsc{QPG}}(s) = \nabla_{\btheta} \left( \sum_a \pi_{\btheta}(s,a) q(s, a; \bw) \right) = \sum_a \nabla_{\btheta} \pi_{\btheta}(s,a; \btheta) q(s, a; \bw).
\]

The RPG gradient $\nabla_{\btheta}^{\textsc{RPG}}(s) = - \nabla_{\btheta} \sum_a \left( q(s,a; \bw) - \sum_b \pi(s,b; \btheta) q(s,b; \bw) \right)^+$
\begin{eqnarray*}
& = & \sum_a \bI \left[ q(s,a; \bw) > \sum_b \pi(s,b; \btheta) q(s,b; \bw) \right] \nabla_{\btheta} \sum_b \pi(s,b; \btheta) q(s,b; \bw)\\
& = & \sum_a \bI \left[ q(s,a; \bw) > \sum_b \pi(s,b; \btheta) q(s,b; \bw) \right]  \sum_b \nabla_{\btheta} \pi(s,b; \btheta) q(s,b; \bw)\\
& = & \sum_a \bI \left[ q(s,a; \bw) > \sum_b \pi(s,b; \btheta) q(s,b; \bw) \right]  \nabla_{\btheta}^{\textsc{QPG}}(s)\\
& = & n_{a+}(s) \nabla_{\btheta}^{\textsc{QPG}}(s),
\end{eqnarray*}
where the first line follows because $\frac{d}{dx}(x)^+ = 0$ for $x < 0$, and
$n_{a+}(s)$ is the number of actions at $s$ with positive advantage. Therefore for any state $s$, the RPG gradient is proportional to the QPG gradient.

In the special case of two-action games, at any state $s$, either the advantages are both 0 or there is one negative-advantage action and one positive-advantage action, so $\nabla_{\btheta}^{\textsc{RPG}}(s) = \nabla_{\btheta}^{\textsc{QPG}}(s)$.

The similarity between RPG and QPG shown here could be (partly) responsible the similar behavior of QPG and RPG that we observed in our experiments.

\section{Additional details on the experiments \label{sec:app-experiments}}
For the policy gradient algorithm experiments, we tried both Adam and SGD optimizers and found SGD to be better performing. We ran sweeps over learning rates and found 0.01 to be the best performing one. The other hyper parameters in the search included $N_q$ and $batch_size$ and we found that the values of 128 and 4 were the best performing in all the games. The entropy cost was swept in the range [0, 0.2] and the best performing value was found to be 0.1 for all the domains. The experiments were run over 5 different seeds and we noted that the exploitability across all 5 seeds was very close (standard deviation = $\pm$0.02).

The Nq and batch size hyper parameters correspond to the number of q-updates (critic) and batch size used to compute actor and critic updates. Note that the critic ($q$) is updated more times (Nq) in order to perform accurate policy evaluation before performing one policy improvement (actor update). The hyper parameters were swept over a grid: $\{16, 32, 64, 128, 256, 512\}$ for Nq and $\{4, 8, 16, 32\}$ for batch size. The best performing values were Nq = 128
and batch size = 4. We performed the experiments over 5 random seeds and found that the results for the chosen hyper
parameters were very close for all 5 seeds (standard deviation of 0.02). While we only performed a simple grid search, we found
that the chosen hyper parameters worked best across all games.

\section{Other Reductions to Counterfactual Regret Minimization
\label{sec:app-reductions}}

The reduction of ACPI to CFR shown in Section~\ref{sec:rpg-seq} is analogous to the reduction from another existing algorithm, sequence-form replicator dynamics (SFRD)~\cite{Gatti13Efficient}, to CFR~\cite{Lanctot14Further}.

In the common setting of estimating gradients from samples, the algorithm then becomes analogous to the on-policy Monte Carlo sampling case in RL.

There is also a model-free sampled version of SFRD called sequence-form Q-learning (SFQ)~\cite{Panozzo14SFQ}. In SFQ, each step samples a deterministic policy requiring time linear in size of the policy. In contrast, actor-critic algorithms can work directly in the behavioral representation (tables indexed by $s,a$). Also, RL-style function approximation can be easily used in the standard way to generalize over the state space.

\section{Negative Results: Monte Carlo Regression CFR (Retracted Baseline)}

In this section, we describe a baseline that we would have liked to include in the comparison: Monte Carlo RCFR, a version of Regression CFR~\cite{Waugh15solving} built from sampled trajectories. Unfortunately, we were unable to get stable results with this algorithm (with Monte Carlo sampling) and more work needs to be done in order to investigate the cause of the instability.

CFR produces, at each iteration $k$, a joint policy $\bpi_k$, and {\it cumulative regrets}
$\textsc{creg}(K, s, a) = \sum_{k \in \{ 1, \cdots, K} \textsc{reg}(\pi_k, s, a)$, 
and average cumulative regrets $\textsc{acreg}(K, s, a) = \frac{1}{K} \textsc{creg}(K, s, a)$.
CFR uses {\it thresholded} cumulative regrets $\textsc{tcreg}(K, s, a) = \textsc{creg}(K, s, a)^+$
(or $\textsc{tacreg}(K, s, a) = \textsc{acreg}(K, s, a)^+$)
values to determine the next policy $\pi_{k+1}$ at each information state using regret matching.

Regression CFR (RCFR) is a policy iteration algorithm that, like CFR, does full tree passes at each iteration.
However, it uses a regressor to approximate the $\textsc{creg}(K, s, a)$ for all information states.
The policies are still derived from these approximate regrets using regret matching.
The original RCFR used input features and regression trees.
Our implementation uses raw input and neural networks, with the same architectures as for our actor-critic
experiments.
We ran some experiments for our implementation of RCFR, which are shown in Figure~\ref{fig:rcfr}.
We found SGD to be unstable for this problem, and had better results using Adam~\cite{KingmaB14}.

\begin{figure}[h!]
\begin{center}
\includegraphics[scale=0.4]{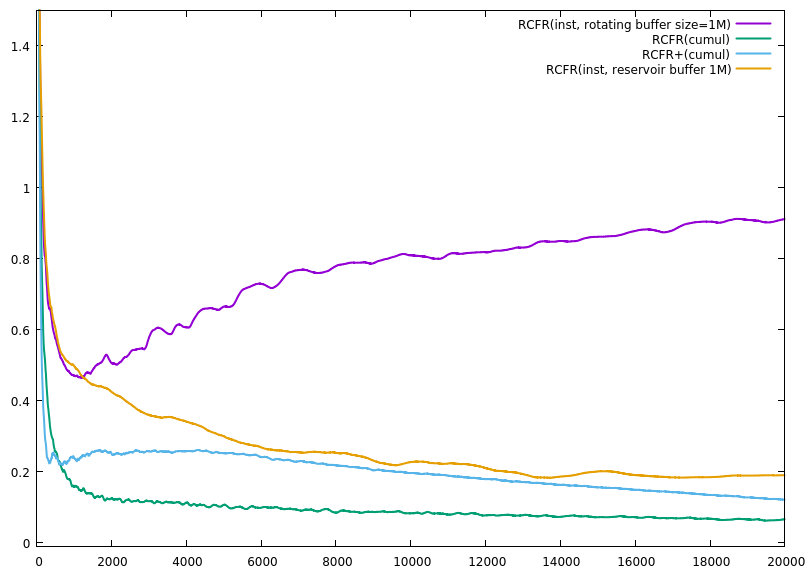}
\end{center}
\caption{Regression CFR convergence results; $x$-axis represents iterations $k$ and $y$-axis represents
$\textsc{NashConv}(\pi_k)$. RCFR(inst, ?) refers to using instantaneous (immediate) regrets $\textsc{reg}$ as the regression targets, whereas RCFR(cumul) refers to the average cumulative regrets $\textsc{acreg}$.
Rotating buffer refers to a circular buffer where new entries delete the oldest entries when the buffer is full.
Reservoir sampling seems to be important when using immediate regret as targets, since it is predicting an
average cumualtive regret. RCFR+ accumulates regret only if the immediate regrets are positive. \label{fig:rcfr}}
\end{figure}

RCFR is not a baseline in our experiments for the same reason CFR is not: it performs policy iteration
which requires free access to the environment, and a full state space sweep per iteration. So, we define a 
sampling version, \defword{Monte Carlo RCFR} (MCRCFR) that is entirely model-free (independent of the environment
and other players' policies). MCRCFR differs from RCFR in the following ways:
\begin{itemize}
\item A value critic is also trained from samples: $v_\pi(s)$. This is done in exactly the same way as in the actor-critic methods.
\item A data set (of size 1 million) is retained to store data.
Each data point stored in this data set is a tuple $(s, R, \widehat{\textsc{reg}}(s,a), a)$"
the input encoding of the information state $s$, the return obtained $R$, the sampled immediate regret
$\widehat{\textsc{reg}}(s,a) = R - v_\pi$, and the action that was chosen $a$.
Reservoir sampling~\cite{Vitter85} is used to replace data in this buffer, so in expectation the data retained in the buffer
is a uniform sample over all the data that was seen.
\item Instead of running a full state space sweep, it samples a trajectory $\rho$ using
an explorative policy $\mu = \epsilon \textsc{Uniform}(\cA(s)) + (1-\epsilon) \pi_i(s)$ at each $s$.
At the beginning, $\epsilon = 1$ and is decayed (multiplied) by 0.995 every 1000 episodes to a minimum of
0.0001.
\item An average policy $\bar{\pi}_i$ is predicted via classification using the actions $a$ that were taken at each $s$ stored in the data set, similar to NFSP. This is the policy we use to assess exploitability.
\end{itemize}

To train, we use Adam~\cite{KingmaB14} a constant learning rate of 0.0001 and batch size of 128. For every 100 sampled episodes, for each network we assembled 10 mini-batches of 128 and run a training step.

MCRCFR seems similar to Advantage Regret Minimization (ARM)~\cite{Jin17ARM}. We outline the differences below:
\begin{itemize}
\item ARM does not maintain a data set: it learns values and (thresholded) advantage values online using moving average of the parameters (whose targets are composed of two separate approximators), but does not predict the average policy.
\item The $q^+_k$ values are bootstrapped from previous values, which is possible when using CFR+. (CFR average cumulative regrets cannot be bootstrapped in this way.) 
\item The $q^+_k$ values predict {\it cumulative sums} rather than average cumulative values (see~\cite[Equation 13]{Jin17ARM}).
\end{itemize}

\section{Corrections to the Original Paper}
\label{app:errata}

The original paper did not include the $O(K)$ terms in statements of Theorem~\ref{thm:pg-cfr-conv} and Theorem~\ref{thm:strong-acpi},
due to concluding that a bound on the weighed regrets implied a bound on the unweighted regrets.
The main problem is that there there is no guarantee about how the negative regrets are weighted compared to the positive regrets.

We attempted to find a fix using stronger guarantees that GIGA has when using linear cost functions, 
such as being no-absolute-regret and no-negative-regret~\cite{kash2019combining,Gofer16} rather than just no-regret, or the fact that the
learning rate is decaying and the algorithm is running in self-play could bound the magnitude of the changes in the policy between steps.
The aim was to see if this could bound the sum of the absolute values of the regrets in a subsequence 
(of either only negative regret or only positive regret).
Despite our best effort, there is a consistent counter-example of a sequence that takes the rough form of a triangular sine wave with decreasing amplitude oscillating
between the positive and negative value extremes; the area under the positive portion (similarly negative portion)-- reflecting the sum of the positive regrets
(and sum of negative regrets, respectively)-- still grows too fast to be bounded by $O(\sqrt{K})$ generally.

We are unaware how often these worst-case sequences occur in practice. An empirical analysis could be an interesting avenue for future work.

\end{document}